\documentclass{article}

\usepackage{ijcai19} % Required by style file

\usepackage{amsmath, amssymb, amsthm, amsfonts} % No amsthm, because it breaks the build. Reason unknown.
\allowdisplaybreaks

\usepackage{algorithm}
\usepackage{bm} % For command \bm{}, which makes argument bold
\usepackage{booktabs} % For formal tables
\usepackage[small]{caption} % Mentioned in style file
\usepackage{paralist} % For inparaenum
\usepackage{graphics}
\usepackage{graphicx}
\usepackage[hidelinks]{hyperref}
\usepackage[utf8]{inputenc}
\usepackage{listings}
\usepackage{enumitem}
\usepackage{mathtools} % For \coloneq (:-)
\usepackage{microtype} % For microtypographical extensions
\usepackage{multirow} % For multirow in booktabs
\usepackage{thmtools, thm-restate} % To restate theorems
\usepackage{times} % Required by style file
\usepackage{url}
\usepackage{multirow}

\usepackage{subcaption}
\usepackage{tikz} % For figures
\usetikzlibrary{backgrounds}
\usetikzlibrary{calc}
\usetikzlibrary{intersections}
\usetikzlibrary{positioning}
\usetikzlibrary{shapes}

\newtheorem{thm}{Theorem}
\declaretheorem[sibling=thm]{problem}

\newcommand{\mlstinline}[1]{\text{\lstinline|#1|}}

\makeatletter
\newcommand{\strong}[1]{\@strong{#1}}
\newcommand{\@@strong}[1]{\textbf{\let\@strong\@@@strong#1}}
\newcommand{\@@@strong}[1]{\textnormal{\let\@strong\@@strong#1}}
\let\@strong\@@strong
\makeatother

\newcommand{\N}{\mathbb{N}}

\newcommand{\union}{\mathbin{\cup}}
\newcommand{\intersection}{\mathbin{\cap}}

\newcommand{\leads}{\coloneq}
\newcommand{\drel}[1]{\mathrel{\textsf{#1}}}
\newcommand{\sdrel}[1]{{\small \drel{#1}}}
\newcommand{\fdrel}[1]{{\footnotesize \drel{#1}}}
\newcommand{\dconst}[1]{\texttt{#1}}
\newcommand{\sdconst}[1]{{\small \dconst{#1}}}
\newcommand{\fdconst}[1]{{\footnotesize \dconst{#1}}}

\newcommand{\Difflog}{{\small \textsc{Difflog}}}

\newcommand{\alps}{{\small \textsc{Alps}}}

\newcommand{\Evaluate}{{\small \textsc{Evaluate}}}

\newcommand\downcast{{\tt downcast}}
\newcommand\andersen{{\tt andersen}}

\newcommand\twocallsite{{\tt 2-call-site}}

\begin{document}
\title{Synthesizing Datalog Programs Using Numerical Relaxation}
\author{
  Xujie Si \and
  Mukund Raghothaman \and
  Kihong Heo \And
  Mayur Naik\\
  \affiliations University of Pennsylvania \\
  \emails \{xsi, rmukund, kheo, mhnaik\}@cis.upenn.edu
}
\maketitle
\begin{abstract}
The problem of learning logical rules from examples arises in diverse fields, including program synthesis, logic
programming, and machine learning. Existing approaches either involve solving computationally difficult combinatorial
problems, or performing parameter estimation in complex statistical models.

In this paper, we present \Difflog, a technique to extend the logic programming language Datalog to the continuous
setting. By attaching real-valued weights to individual rules of a Datalog program, we naturally associate numerical
values with individual conclusions of the program. Analogous to the strategy of numerical relaxation in optimization
problems, we can now first determine the rule weights which cause the best agreement between the training labels and the
induced values of output tuples, and subsequently recover the classical discrete-valued target program from the
continuous optimum.

We evaluate \Difflog{} on a suite of 34~benchmark problems from recent literature in knowledge discovery, formal
verification, and database query-by-example, and demonstrate significant improvements in learning complex programs with
recursive rules, invented predicates, and relations of arbitrary arity.
\end{abstract}

% Refines the semantics of the classical program
% Relaxations
% Resulting formalism: simple to evaluate, ...

% An important problem in automated reasoning involves learning logical patterns from structured data. Existing approaches
% to this task of inductive logic programming either involve solving computationally difficult combinatorial problems or
% performing parameter estimation in complex statistical models. In this paper, we present \Difflog, a simple
% technique to extend the traditional discrete-valued semantics of logic programming languages to the continuous domain.
% By attaching real-valued weights to individual inference rules, we naturally associate numerical values with individual
% conclusions of the logic program. Rule learning may then be cast as the problem of determining the values of the weights
% which cause the best agreement between training labels and induced values of output tuples. We propose a novel
% algorithmic framework to efficiently evaluate \Difflog{} programs with provenance information, which in turn makes it
% feasible to employ standard numerical techniques such as Newton's method to the synthesis of logic programs. On a
% comprehensive suite of 34 benchmark problems from different domains, \Difflog{} can learn complex programs with
% recursive rules and relations of arbitrary arity.

\section{Introduction}
\label{sec:Intro}

%%%%%%%%%%%%%%%%%%%%%%%%%%%%%%%%%%%%%%%%%%%%%%%%%%%%%%%%%%%%%%%%%%%%%%%%%%%%%%%%%%%%%%%%%%%%%%%%%%%%%%%%%%%%%%%%%%%%%%%%
% 1. Motivation

As a result of its rich expressive power and efficient implementations,
the logic programming language Datalog has witnessed applications in diverse domains such as bioinformatics~\cite{Seo2018},
big-data analytics~\cite{Shkapsky2016}, robotics~\cite{Poole1995}, networking~\cite{Loo2006}, and
formal verification~\cite{Bravenboer2009}. Users on the other hand are often unfamiliar with
logic programming.  The programming-by-example (PBE) paradigm aims to bridge this gap by
providing an intuitive interface for non-expert users~\cite{Gulwani2011}.

%%%%%%%%%%%%%%%%%%%%%%%%%%%%%%%%%%%%%%%%%%%%%%%%%%%%%%%%%%%%%%%%%%%%%%%%%%%%%%%%%%%%%%%%%%%%%%%%%%%%%%%%%%%%%%%%%%%%%%%%
% 2. Problem Description + Background

Typically, a PBE system is given a set of input tuples and sets of desirable and undesirable output tuples.
The central computational problem is that of synthesizing a Datalog program, i.e., a set of logical inference
rules which produces, from the input tuples, a set of conclusions which is compatible with the output tuples.
Previous approaches to this problem focus on optimizing the combinatorial exploration of the search space.
For example, \alps{} maintains a small set of syntactically most-general and most-specific candidate
programs~\cite{ALPS}, Zaatar encodes the derivation of output tuples as a SAT formula for subsequent solving by
a constraint solver~\cite{Zaatar}, and inductive logic programming (ILP) systems employ sophisticated pruning
algorithms based on ideas such as inverse entailment~\cite{Progol}.
Given the computational complexity of the search problem, however, these systems are hindered by large
or difficult problem instances.
Furthermore, these systems have difficulty coping with minor user errors or noise in the training data.

%%%%%%%%%%%%%%%%%%%%%%%%%%%%%%%%%%%%%%%%%%%%%%%%%%%%%%%%%%%%%%%%%%%%%%%%%%%%%%%%%%%%%%%%%%%%%%%%%%%%%%%%%%%%%%%%%%%%%%%%
% 3. Our Ideas

% TODO: Remark that Difflog can be evaluated in polynomial time, in contrast to probabilistic frameworks such as
% MLNs, ProbLog, probabilistic soft logic, etc.

In this paper, we take a fundamentally different approach to the problem of synthesizing Datalog programs.
Inspired by the success of numerical methods in machine learning and other large scale optimization problems,
and of the strategy of relaxation in solving combinatorial problems such as integer linear programming, we extend the
classical discrete semantics of Datalog 
%(i.e., a tuple is either derived, or it is not) 
to a continuous setting named \Difflog, where each rule is annotated with a real-valued weight, and the program 
computes a numerical value for each output tuple.
This step can be viewed as an instantiation of the general $K$-relation framework for database provenance 
\cite{SemiringProvenance} with the Viterbi semiring being chosen as the underlying space $K$ of provenance tokens.
We then formalize the program synthesis problem as that of selecting a subset of target rules from a large set of
candidate rules, and thereby uniformly capture various methods of inducing syntactic bias, including
syntax-guided synthesis (SyGuS)~\cite{SyGuS}, and template rules in meta-interpretive learning~\cite{Metagol}.

The synthesis problem thus reduces to that of finding the values of the rule weights which result in the best agreement
between the computed values of the output tuples and their specified values ($1$ for desirable and $0$ for undesirable tuples).
The fundamental NP-hardness of the underlying decision problem manifests as a complex search surface,
with local minima and saddle points.
To overcome these challenges, we devise a hybrid optimization algorithm which combines Newton's root-finding method
with periodic invocations of a simulated annealing search.
Finally, when the optimum value is reached, connections between the semantics of \Difflog{} and Datalog
enable the recovery of a classical discrete-valued Datalog program from the continuous-valued optimum produced
by the optimization algorithm.

%%%%%%%%%%%%%%%%%%%%%%%%%%%%%%%%%%%%%%%%%%%%%%%%%%%%%%%%%%%%%%%%%%%%%%%%%%%%%%%%%%%%%%%%%%%%%%%%%%%%%%%%%%%%%%%%%%%%%%%%
% 4. Experiments

% TODO: Elaborate on experiments
% TODO: Xujie, Mayur: Please make sure that I am not making unsupported claims here.
A particularly appealing aspect of relaxation-based synthesis is the randomness caused by the choice of the starting
position and of subsequent Monte Carlo iterations. This manifests both as a variety of different solutions to the same
problem, and as a variation in running times. Running many search instances in parallel therefore enables stochastic
speedup of the synthesis process, and allows us to leverage compute clusters in a way that is fundamentally impossible
with deterministic approaches. We have implemented \Difflog{} and evaluate it on a suite of 34~benchmark programs from recent
literature. We demonstrate significant improvements over the state-of-the-art, even while synthesizing complex programs
with recursion, invented predicates, and relations of arbitrary arity.

%%%%%%%%%%%%%%%%%%%%%%%%%%%%%%%%%%%%%%%%%%%%%%%%%%%%%%%%%%%%%%%%%%%%%%%%%%%%%%%%%%%%%%%%%%%%%%%%%%%%%%%%%%%%%%%%%%%%%%%%
% Contributions

\paragraph{Contributions.\!\!}
Our work makes the following contributions:
\begin{enumerate}[leftmargin=*,itemsep=0pt]
\item A formulation of the Datalog synthesis problem as that of selecting a set of desired rules.
This formalism generalizes syntax-guided query synthesis and meta-rule guided search.
\item A fundamentally new approach to solving rule selection by numerically minimizing the difference between the
  weighted set of candidate rules and the reference output.
\item An extension of Datalog which also associates output tuples with numerical weights, and which is a continuous
  refinement of the classical semantics.
\item Experiments showing state-of-the-art performance on a suite of diverse benchmark programs from
  recent literature.
\end{enumerate}

\section{Related Work}
\label{sec:Related}

% TODO: Consult papers from AAAI 2018/2019, IJCAI 2018, and NIPS 2018 for new related work
% 1. Probabilistic frameworks, including MLNs, ProbLog, etc. Structure learning. Highlight #P-completeness of
%    evaluation. Perhaps a sentence or two talking about NeuralLP, delta-ILP, etc.

% 2. ILP frameworks, including Metagol, Progol, etc. Highlight that our framework is general, not limited to predicates
%    of arity 2, supports recursion out of the box, invented predicates, relations of arbitrary arity, noise.

% 3. Synthesis frameworks including ALPS, Zaatar, and some background, such as SyGuS and FlashFill.

% 4. Technical aspects: relaxed semantics for Datalog, MCMC in program synthesis (STOKE), etc.

\paragraph{Weighted logical inference.\!\!}
The idea of extending logical inference with weights has been studied by the community in statistical relational learning.
Markov Logic Networks~\cite{MLN,LearnMLN} view a first order formula as a template for generating a Markov random field, where the weight attached to the formula specifies the likelihood of its grounded clauses.
%It concerns the inference of the most likely world of output tuples rather than learning rules used for deriving output tuples.
ProbLog~\cite{ProbLog} extends Prolog with probabilistic rules and reduces the inference problem to weighted model counting.
DeepProbLog~\cite{DeepProbLog} further extends ProbLog with neural predicates (e.g., input data can be images).
%\XS{The following sentence is very critical, though we might put it in a different place.}
These frameworks could potentially serve as the central inference component of our framework but we use the
Viterbi semiring due to two critical factors:
\begin{inparaenum}[(\itshape a\upshape)]
\item the exact inference problem of these frameworks is \#P-complete whereas inference in the Viterbi semiring is polynomial; and
\item automatic differentiation cannot easily be achieved without significantly re-engineering these frameworks.
\end{inparaenum}

\paragraph{Inductive logic programming (ILP).\!\!}
The Datalog synthesis problem can be seen as an instance of the classic ILP problem.
Metagol~\cite{Metagol} supports higher-order dyadic Datalog synthesis but the synthesized program can only consist of relations of arity two.
%Though higher-order dyadic Datalog program has universal Turing expressivity,
%it is not clear reduce the given training input and output relations with arbitrary arity into relations accepted by Metagol.
Metagol is built on top of Prolog which limits its scalability, and also introduces issues with non-termination, especially when
predicates have not already been partially ordered by the user.
In contrast, ALPS~\cite{ALPS} builds on top of Z3 fixed point engine and exhibits much better scalability.
Recent works such as NeuralLP~\cite{NeuralLP} and $\partial$ILP~\cite{diffILP} reduce Datalog program synthesis to a
differentiable end-to-end learning process by modeling relation joins as matrix multiplication, which also limits them
to relations of arity two.
NTP~\cite{NTP} constructs a neural network as a learnable proof (or derivation) for each output tuple up to a predefined depth (e.g. $\leq 2$) with a few (e.g. $\leq 4$) templates,
where the neural network could be exponentially large when either the depth or the number of templates grows.
The predefined depth and a small number of templates could significantly limit the richness of learned programs.
Our work seeks to synthesize Datalog programs consisting of relations of arbitrary arity and support rich features like
recursion and predicate invention.

\paragraph{MCMC methods for program synthesis.\!\!}
Markov chain Monte-Carlo (MCMC) methods have also been used for program synthesis.
For example, in STOKE, \cite{STOKE} apply the Metropolis-Hastings algorithm to synthesize efficient loop free programs.
Similarly, \cite{liang:icml10} show that program transformations can be efficiently learned from demonstrations by MCMC
inference.

%Since their introduction in the 1940s~\cite{MCMCHistory}, Monte Carlo methods have found wide applications in the
%physical sciences, statistics, and optimization.

%The STOKE project of Schufza et al.~\cite{STOKE} is
%an important application of these techniques to compiler optimization and tuning floating-point programs. As
%traditionally presented, the Metropolis Hastings algorithm is a sampling-based approach to generate points from a space
%according to some probability distribution.
% Extension to optimization
%\begin{enumerate}
%\item Optimization by simulated annealing. Kirkpatrick, Gelatt Jr, and Vecchi. Science, 1983.
%\item Simulated annealing: A proof of convergence. Granville, Krivanek, and Rason. IEEE Pattern analysis and Machine
%  intelligence, 1994.
%\end{enumerate}

\section{The Datalog Synthesis Problem}
\label{sec:Problem}

% TODO: Motivate challenges in example problem instance:
% 1. Recursion
% 2. Partially labelled data
% 3. Invented predicates
% 4. Arbitrary rule patterns
% 5. Arbitrary predicate arity
% 6. Possible noise in training data

%%%%%%%%%%%%%%%%%%%%%%%%%%%%%%%%%%%%%%%%%%%%%%%%%%%%%%%%%%%%%%%%%%%%%%%%%%%%%%%%%%%%%%%%%%%%%%%%%%%%%%%%%%%%%%%%%%%%%%%%
% 1. Example

In this section, we concretely describe the Datalog synthesis problem, and establish some basic complexity results.
 We use the family tree shown in Figure~\ref{fig:Problem:FamilyTree} as a running example.
%The input data consists of a number of tuples, each of the form $\sdrel{parent}(x, y)$, describing the fact that $x$ is a parent of $y$
%in the family being considered. In this example, the user has a set of input tuples $\sdrel{parent}(x, y)$, and wishes
%to determine which persons belong to the same generation in the family.
In Section~\ref{sub:Problem:Background}, we briefly describe how one may compute $\sdrel{samegen}(x, y)$ from
$\sdrel{parent}(x, y)$ using a Datalog program. In Section~\ref{sub:Problem:RuleSelection}, we formalize the
query synthesis problem as that of rule selection.

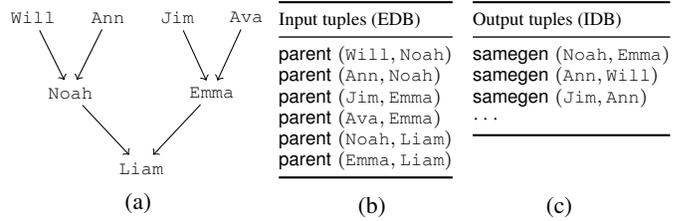
\begin{figure}
  \centering
  \hspace{-0.1in}\begin{subfigure}{0.2\textwidth}
{\scriptsize
    \begin{tikzpicture}
    \node (will) {\dconst{Will}};
    \node [right=0.3 of will] (ann) {\dconst{Ann}};
    \node [right=0.3 of ann] (james) {\dconst{Jim}};
    \node [right=0.3 of james] (ava) {\dconst{Ava}};

    \node (noah) at ($(will)!0.5!(ann) + (0, -1)$) {\dconst{Noah}};
    \draw [->] (will) -- (noah);
    \draw [->] (ann) -- (noah);

    \node (emma) at ($(james)!0.5!(ava) + (0, -1)$) {\dconst{Emma}};
    \draw [->] (james) -- (emma);
    \draw [->] (ava) -- (emma);

    \node (liam) at ($(noah)!0.5!(emma) + (0, -1)$) {\dconst{Liam}};
    \draw [->] (noah) -- (liam);
    \draw [->] (emma) -- (liam);
    \end{tikzpicture}
}
	\vspace{-0.18in}
  \caption{}
  \label{sfig:Problem:FamilyTree:Graph}
  \end{subfigure}
  %\hspace*{\fill}
  \begin{subfigure}{0.14\textwidth}
    \scriptsize
    \begin{tabular}{@{}l@{}} \toprule
    Input tuples (EDB) \\ \midrule
    $\drel{parent}(\dconst{Will}, \dconst{Noah})$ \\
    $\drel{parent}(\dconst{Ann}, \dconst{Noah})$ \\
    $\drel{parent}(\dconst{Jim}, \dconst{Emma})$ \\
    $\drel{parent}(\dconst{Ava}, \dconst{Emma})$ \\
    $\drel{parent}(\dconst{Noah}, \dconst{Liam})$ \\
    $\drel{parent}(\dconst{Emma}, \dconst{Liam})$ \\
    \bottomrule \end{tabular}
  \caption{}
  \label{sfig:Problem:FamilyTree:EDB}
  \end{subfigure}
  \begin{subfigure}{0.13\textwidth}
    \scriptsize
    \begin{tabular}{@{}l@{}} \toprule
    Output tuples (IDB) \\ \midrule
    $\drel{samegen}(\dconst{Noah}, \dconst{Emma})$ \\
    $\drel{samegen}(\dconst{Ann}, \dconst{Will})$ \\
    $\drel{samegen}(\dconst{Jim}, \dconst{Ann})$ \\
    $\cdots$ \\
    \bottomrule \\[8pt]\end{tabular}
  \caption{}
  \label{sfig:Problem:FamilyTree:IDB}
  \end{subfigure}
  \hspace*{\fill}
\vspace{-0.1in}
\caption{Example of a family tree~(a), and its representation as a set of input
  tuples~(b). An edge from $x$ to $y$ indicates that $x$ is a parent of $y$, and is
  represented symbolically as the tuple $\sdrel{parent}(x, y)$. The user wishes to realize the relation
  $\sdrel{samegen}(x, y)$, indicating the fact that $x$ and $y$ occur are from the same generation of the
  family~(c).}
\vspace{-0.1in}
\label{fig:Problem:FamilyTree}
\end{figure}

%%%%%%%%%%%%%%%%%%%%%%%%%%%%%%%%%%%%%%%%%%%%%%%%%%%%%%%%%%%%%%%%%%%%%%%%%%%%%%%%%%%%%%%%%%%%%%%%%%%%%%%%%%%%%%%%%%%%%%%%
% 2. Classical Computation with Datalog

\subsection{Overview of Datalog}
\label{sub:Problem:Background}

The set of tuples inhabiting relation $\sdrel{samegen}(x, y)$ can be computed using the following pair of
\emph{inference rules}, $r_1$ and $r_2$:\linebreak
\vspace{-0.15in}
\begin{alignat*}{1}
  r_1\!\!: \sdrel{samegen}(x, y) & \leads \sdrel{parent}(x, z), \sdrel{parent}(y, z). \\ %\label{eq:Problem:SGenBase} \\
  r_2\!\!: \sdrel{samegen}(x, u) & \leads \sdrel{parent}(x, y), \sdrel{parent}(u, v), \sdrel{samegen}(y, v).  %\label{eq:Problem:SGenInd}
\end{alignat*}
Rule $r_1$ describes the fact that for all persons $x$, $y$, and $z$, if both $x$ and $y$ are parents of $z$,
then $x$ and $y$ occur at the same level of the family tree. Informally, this rule forms the base of the inductive
definition. Rule $r_2$ forms the inductive step of the definition, and provides that $x$ and $u$ occur in the
same generation whenever they have children $y$ and $v$ who themselves occur in the same generation.

By convention, the relations which are explicitly provided as part of the input are called the EDB, $\mathcal{I} = \{
\sdrel{parent} \}$, and those which need to be computed as the output of the program are called the IDB, $\mathcal{O} =
\{ \sdrel{samegen} \}$. To evaluate this program, one starts with the set of input tuples, and repeatedly applies
rules~$r_1$ and $r_2$ to derive new output tuples. Note that because of the
appearance of the literal $\sdrel{samegen}(y, v)$ on the right side of rule~$r_2$, discovering a
single output tuple may recursively result in the further discovery of additional output tuples. The derivation process
ends when no additional output tuples can be derived, i.e., when the set of conclusions reaches a \emph{fixpoint}.

More generally, we assume a collection of \emph{relations}, $\{P, Q, \dots\}$. Each relation $P$ has an arity $k \in
\N$, and is a set of \emph{tuples}, each of which is of the form $P(c_1, c_2, \dots, c_k)$, for some \emph{constants}
$c_1$, $c_2$, \dots, $c_k$. The Datalog program is a collection of rules, where each rule $r$ is of the form:
\[
  P_h(\bm{u}_h) \leads P_1(\bm{u}_1), P_2(\bm{u}_2), \dots, P_k(\bm{u}_k),
\]
where $P_h$ is an output relation, and $\bm{u}_h$, $\bm{u}_1$, $\bm{u}_2$, \dots, $\bm{u}_k$ are vectors of
\emph{variables} of appropriate length. The variables $\bm{u}_1$, $\bm{u}_2$, \ldots, $\bm{u}_k$, $\bm{u}_h$ appearing
in the rule are implicitly universally quantified, and instantiating them with appropriate constants $\bm{v}_1$,
$\bm{v}_2$, \ldots, $\bm{v}_k$, $\bm{v}_h$ yields a grounded constraint $g$ of the form $P_1(\bm{v}_1) \land
P_2(\bm{v}_2) \land \dots \land P_k(\bm{v}_k) \implies P_h(\bm{v}_h)$: \emph{``If all of the antecedent tuples $A_g =
\{ P_1(\bm{v}_1), P_2(\bm{v}_2), \dots, P_k(\bm{v}_k) \}$ are derivable, then the conclusion $c_g = P_h(\bm{v}_h)$ is
also derivable.''}

%%%%%%%%%%%%%%%%%%%%%%%%%%%%%%%%%%%%%%%%%%%%%%%%%%%%%%%%%%%%%%%%%%%%%%%%%%%%%%%%%%%%%%%%%%%%%%%%%%%%%%%%%%%%%%%%%%%%%%%%
% 3. The Rule Selection Problem

\subsection{Synthesis as Rule Selection}
\label{sub:Problem:RuleSelection}

%%%%
% 3a. Explaining the synthesis problem instance

\paragraph{The input-output examples, $I$, $O_+$, and $O_-$.\!\!}
Instead of explicitly providing rules~$r_1$ and $r_2$, the user provides an
example instance of the EDB $I$, and labels a few tuples of the output relation as ``desirable'' or ``undesirable''
respectively:
\[
  O_+ = \{ \sdrel{samegen}(\sdconst{Ann}, \sdconst{Jim}) \}, \text{ and}
\]
\[
  O_- = \{ \sdrel{samegen}(\sdconst{Ava}, \sdconst{Liam}), \sdrel{samegen}(\sdconst{Jim}, \sdconst{Emma}) \},
\]
indicating that $\sdconst{Ann}$ and $\sdconst{Jim}$ are from the same generation, but $\sdconst{Ava}$ and
$\sdconst{Liam}$ and $\sdconst{Jim}$ and $\sdconst{Emma}$ are not. Note that the user is free to label as many
potential output tuples as they wish, and the provided labels $O_+ \union O_-$ need not be exhaustive. The goal of
the program synthesizer is to find a set of rules $R_s$ which produce all of the desired output tuples, i.e., $O_+ \subseteq
R_s(I)$, and none of the undesired tuples, i.e., $O_- \intersection R_s(I) = \emptyset$.

%%%%
% 3b. Syntactic bias: Set of candidate rules
%     - Motivate by reference to meta-rules, syntax-guidance, etc. Highlight that order is not necessary, unlike in
%       Metagol.
%     - Explain how we bootstrap the set of candidate rules using augmentation.

\paragraph{The set of candidate rules, $R$.\!\!}
The user often possesses additional information about the problem instance and the concept being targeted. This
information can be provided to the synthesizer through various forms of \emph{bias}, which direct the search towards
desired parts of the search space. A particularly common form in the recent literature on program synthesis is syntactic: for
example, SyGuS requires a description of the space of potential solution programs as a context-free grammar~%
\cite{SyGuS}, and recent ILP systems such as Metagol \cite{Metagol} require the user to provide a set of higher-order rule templates
(``\emph{metarules}'') and order constraints over predicates and variables that appear in clauses. In this paper, we
assume that the user has provided a large set of candidate rules $R$ and that the target concept $R_s$ is a subset of
these rules: $R_s \subseteq R$.

These candidate rules can express various patterns that could conceivably discharge the problem instance. For example,
$R$ can include the candidate rule $r_s$, ``$\sdrel{samegen}(x, y) \leads \sdrel{samegen}(y, x)$'', which indicates that
the output relation is symmetric, and the candidate rule $r_t$, ``$\sdrel{samegen}(x, z) \leads \sdrel{samegen}(x, y),
\sdrel{samegen}(y, z)$'', which indicates that the relation is transitive. Note that the assumption of the candidate
rule set $R$ uniformly subsumes many previous forms of syntactic bias, including those in SyGuS and Metagol.

Also note that $R$ can often be automatically populated: In our experiments in Section~\ref{sec:Experiments}, we
automatically generate $R$ using the approach introduced by \alps~\cite{ALPS}. We start with seed rules
that follow a simple chain pattern (e.g., ``$P_1(x_1, x_4) \leads P_2(x_1, x_2), P_3(x_2, x_3), P_4(x_3, x_4)$''),
and repeatedly augment $R$ with simple edits to the variables, predicates, and literals of current candidate rules. The
candidate rules thus generated exhibit complex patterns, including recursion, and contain literals of arbitrary arity.
Furthermore, any conceivable Datalog rule can be produced with a sufficiently large augmentation distance.

%%%%
% 3c. Formally state the rule selection problem

\begin{problem}[Rule Selection]
\label{prob:Problem}
Let the following be given:
\begin{inparaenum}[(\itshape a\upshape)]
\item a set of input relations, $\mathcal{I}$ and output relations, $\mathcal{O}$,
\item the set of input tuples $I$,
\item a set of positive output tuples $O_+$,
\item a set of negative output tuples $O_-$, and
\item a set of candidate rules $R$ which map the input relations $\mathcal{I}$ to the output relations $\mathcal{O}$.
\end{inparaenum}
Find a set of target rules $R_s \subseteq R$ such that:
\[
  O_+  \subseteq R_s(I), \text{ and ~}   O_-  \intersection R_s(I) = \emptyset.
\]
\end{problem}

%%%%
% 3d. Complexity

Finally, we note that the rule selection problem is NP-hard: this is because multiple rules in the target program $R_s$
may interact in non-compositional ways. The proof proceeds through a straightforward encoding of the satisfiability of a
3-CNF formula, and is provided in the Appendix.

\begin{restatable}{thm}{thmProblemComplexity}
\label{thm:Problem:Complexity}
Determining whether an instance of the rule selection problem, $(\mathcal{I}, \mathcal{O}, I, O_+, O_-, R)$, admits a
solution is NP-hard.
\end{restatable}

\section{A Smoothed Interpretation for Datalog}
\label{sec:Framework}

In this section, we describe the semantics of \Difflog, and present an algorithm to evaluate and automatically
differentiate this continuous-valued extension.

%%%%%%%%%%%%%%%%%%%%%%%%%%%%%%%%%%%%%%%%%%%%%%%%%%%%%%%%%%%%%%%%%%%%%%%%%%%%%%%%%%%%%%%%%%%%%%%%%%%%%%%%%%%%%%%%%%%%%%%%
% 1. Relaxing Rule Selection

\subsection{Relaxing Rule Selection}
\label{sub:Framework:Relaxation}

The idea motivating \Difflog{} is to generalize the concept of rule selection: instead of a set of binary decisions, we
associate each rule $r$ with a numerical weight $w_r \in [0, 1]$. One possible way to visualize these weights is as the
extent to which they are present in the current candidate program. The central challenge, which we will now address, is
in specifying how the vector of rule weights $\bm{w}$ determine the numerical values $v_t^{R,I}(\bm{w})$ for the output tuples
$t$ of the program. We use notation $v_t(\bm{w})$ when the set of rules $R$ and the set of input tuples $I$
are evident from context.

Every output tuple of a Datalog program is associated with a set of derivation trees, such as those shown in Figure~%
\ref{fig:Framework:Relaxation:Trees}. Let $r_g$ be the rule associated with each instantiated clause $g$ that appears in
the derivation tree $\tau$. We define the value of $\tau$, $v_\tau(\bm{w})$, as the product of the weights of all
clauses appearing in $\tau$, and the value of an output tuple $t$ as being the supremum of the values of all derivation
trees of which it is the conclusion:
\begin{alignat}{1}
  v_\tau(\bm{w})  = \prod_{\text{clause } g \in \tau} w_{r_g}, \text{ and}\ \ \ \ \ \  \label{eq:Framework:Relaxation:Tree} \\
  v_t(\bm{w})  = \sup_{\tau \text{ with conclusion } t} v_\tau(\bm{w}), \label{eq:Framework:Relaxation:Tuple}
\end{alignat}
with the convention that $\sup(\emptyset) = 0$. For example, if $w_{r_1} = 0.8$ and $w_{r_2} = 0.6$, then the weight of
the trees $\tau_1$ and $\tau_2$ from Figure~\ref{fig:Framework:Relaxation:Trees} are respectively $v_{\tau_1}(\bm{w}) =
w_{r_1} = 0.8$ and $v_{\tau_2}(\bm{w}) = w_{r_1} w_{r_2} = 0.48$.

\begin{figure}
  \hspace{-0.2in}
  \begin{subfigure}[b]{0.25\textwidth}
  \resizebox{\textwidth}{!}{
\newcommand{\derivtreesep}{0.6}
\newcommand{\derivtreeseptwo}{1.3}
\setlength{\fboxsep}{0pt}%
\setlength{\fboxrule}{0pt}%
\framebox{
\begin{tikzpicture}[tuple/.style={draw}, % % rounded rectangle
                    edb/.style={tuple, fill=black!15}, %
                    outTuple/.style={tuple, double}, %
                    clause/.style={},
                    scale=0.85,
                    every node/.style={scale=0.85}]
  \node [edb] (pwn) {$\fdrel{parent}(\fdconst{Will}, \fdconst{Noah})$};
  \node [edb, right=0.3 of pwn] (pan) {$\fdrel{parent}(\fdconst{Ann}, \fdconst{Noah})$};
  \node [clause] (r1wan) at ($(pwn)!0.5!(pan) + (0, -0.8)$) {$r_1(\fdconst{Will}, \fdconst{Ann}, \fdconst{Noah})$};
  \node [tuple] (swa) at ($(r1wan) + (0, -0.8)$) {$\fdrel{samegen}(\fdconst{Will}, \fdconst{Ann})$};
  \begin{pgfonlayer}{background}
  \draw [->] (pwn) -- (r1wan);
  \draw [->] (pan) -- (r1wan);
  \draw [->] (r1wan) -- (swa);
  \end{pgfonlayer}

\end{tikzpicture}
}
}
  \vspace{-0.15in}
  \caption{}
  \label{sfig:Framework:Relaxation:Trees:1}
  \end{subfigure}
  \begin{subfigure}[b]{0.25\textwidth}
  \resizebox{\textwidth}{!}{
\newcommand{\derivtreesep}{0.6}
\newcommand{\derivtreeseptwo}{1.3}
\setlength{\fboxsep}{0pt}%
\setlength{\fboxrule}{0pt}%
\framebox{
\begin{tikzpicture}[tuple/.style={draw}, % % rounded rectangle
                    edb/.style={tuple, fill=black!15}, %
                    outTuple/.style={tuple, double}, %
                    clause/.style={},
                    scale=0.85,
                    every node/.style={scale=0.85}]
  \node [edb] (pnl1) {$\fdrel{parent}(\fdconst{Noah}, \fdconst{Liam})$};
  \node [edb, right=0.3 of pnl1] (pnl2) {$\fdrel{parent}(\fdconst{Noah}, \fdconst{Liam})$};
  \node [clause] (r1nnl) at ($(pnl1)!0.5!(pnl2) + (0, -0.8)$) {$r_1(\fdconst{Noah}, \fdconst{Noah}, \fdconst{Liam})$};
  \node [tuple] (snn) at ($(r1nnl) + (0, -0.8)$) {$\fdrel{samegen}(\fdconst{Noah}, \fdconst{Noah})$};
  \begin{pgfonlayer}{background}
  \draw [->] (pnl1) -- (r1nnl);
  \draw [->] (pnl2) -- (r1nnl);
  \draw [->] (r1nnl) -- (snn);
  \end{pgfonlayer}

  \node [edb] (pwn) at ($(pnl1 |- snn) + (0, -0.8)$) {$\sdrel{parent}(\fdconst{Will}, \fdconst{Noah})$};
  \node [edb] (pan) at ($(pwn) + (3.3, 0)$) {$\sdrel{parent}(\fdconst{Ann}, \fdconst{Noah})$};
  \node [clause] (r2wnan) at ($(pan -| snn) + (0, -0.8)$) {$r_2(\fdconst{Will}, \fdconst{Noah}, \fdconst{Ann}, \fdconst{Noah})$};
  \node [tuple] (swa) at ($(r2wnan) + (0, -0.8)$) {$\fdrel{samegen}(\fdconst{Will}, \fdconst{Ann})$};
  \begin{pgfonlayer}{background}
  \draw [->] (pwn) -- (r2wnan);
  \draw [->] (pan) -- (r2wnan);
  \draw [->] (snn) -- (r2wnan);
  \draw [->] (r2wnan) -- (swa);
  \end{pgfonlayer}
\end{tikzpicture}
}
}
  \vspace{-0.15in}
  \caption{}
  \label{sfig:Framework:Relaxation:Trees:2}
  \end{subfigure}
\vspace{-0.2in}
\caption{Examples of derivation trees, $\tau_1$~(a) and
  $\tau_2$~(b) induced by various combinations of candidate rules, applied to the
  EDB of familial relationships from Figure~\ref{fig:Problem:FamilyTree}. The input tuples are shaded in grey. We
  present two derivation trees for the conclusion $\sdrel{samegen}(\sdconst{Will}, \sdconst{Ann})$ using rules $r_1$ and
  $r_2$  in Section~\ref{sub:Problem:Background}.}
\label{fig:Framework:Relaxation:Trees}
\end{figure}
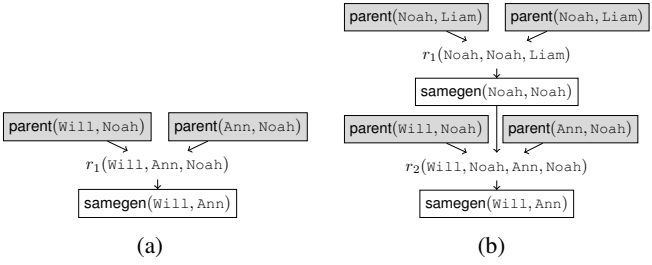

Since $0 \leq w_r \leq 1$, it follows that $v_\tau(\bm{w}) \leq 1$. Also note that a single output tuple may be the
conclusion of infinitely many proof trees (see the derivation structure in Figure~%
\ref{fig:Framework:Relaxation:SymCycle}), leading to the deliberate choice of the supremum in Equation~%
\ref{eq:Framework:Relaxation:Tuple}.

\begin{figure}
\centering
\vspace{-0.05in}
\resizebox{0.35\textwidth}{!}{
\newcommand{\derivtreesep}{0.6}
\newcommand{\derivtreeseptwo}{1.3}
\setlength{\fboxsep}{0pt}%
\setlength{\fboxrule}{0pt}%
\framebox{
\begin{tikzpicture}[tuple/.style={draw}, % % rounded rectangle
                    edb/.style={tuple, fill=black!15}, %
                    outTuple/.style={tuple, double}, %
                    clause/.style={},
                    scale=0.85,
                    every node/.style={scale=0.85}]
  \node [tuple] (swa) {$\fdrel{samegen}(\fdconst{Will}, \fdconst{Ann})$};
  \node [tuple, right=1 of swa] (saw) {$\fdrel{samegen}(\fdconst{Ann}, \fdconst{Will})$};
  \node [clause] (rsaw) at ($(swa)!0.5!(saw) + (0, 1)$) {$r_s(\fdconst{Ann}, \fdconst{Will})$};
  \node [clause] (rswa) at ($(swa)!0.5!(saw) + (0, -1)$) {$r_s(\fdconst{Will}, \fdconst{Ann})$};
  \node [above=0.5 of swa] (swaroot) {$\cdots$};
  \node [above=0.5 of saw] (sawroot) {$\cdots$};
  \begin{pgfonlayer}{background}
  \draw [->] (swaroot) -- (swa);
  \draw [->] (sawroot) -- (saw);
  \draw [->] (swa) -- (rsaw);
  \draw [->] (rsaw) -- (saw);
  \draw [->] (saw) -- (rswa);
  \draw [->] (rswa) -- (swa);
  \end{pgfonlayer}

\end{tikzpicture}
}
}
\caption{The rule $r_s$, ``$\sdrel{someone}(x, y) \leads \sdrel{samegen}(y, x)$'', induces cycles in the clauses
  obtained at fixpoint. When unrolled into derivation trees such as those in
  Figure~\ref{fig:Framework:Relaxation:Trees}, these cycles result in the production of infinitely many derivation trees
  for a single output tuple.\allowbreak}
\vspace{-0.25in}
\label{fig:Framework:Relaxation:SymCycle}
\end{figure}
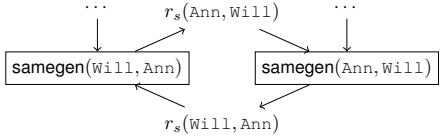

One way to consider Equations~\ref{eq:Framework:Relaxation:Tree} and~\ref{eq:Framework:Relaxation:Tuple} is as replacing
the traditional operations $(\land, \lor)$ and values $\{ \text{true}, \text{false} \}$ of the Boolean semiring with the
corresponding operations $(\times, \max)$ and values $[0, 1]$ of the Viterbi semiring. The study of various
semiring interpretations of database query formalisms has a rich history motivated by the idea of \emph{data
provenance}.
% We will ourselves use this idea to enable automatic program differentiation in Section~\ref{sub:Framework:Evaluation}.
The following result follows from Prop.~5.7 in~\cite{SemiringProvenance}, and concretizes the idea that \Difflog{}
is a refinement of Datalog:
\begin{restatable}{thm}{thmFrameworkRelaxationRefinement}
\label{thm:Framework:Relaxation:Refinement}
Let $R$ be a set of candidate rules, and $\bm{w}$ be an assignment of weights $w_r \in [0, 1]$ to each of them, $r \in
R$. Define $R_s = \{ r \mid w_r \gneq 0 \}$, and consider a potential output tuple $t$. Then, $v_t^{R,I}(\bm{w}) \gneq 0$ iff
$t \in R_s(I)$.
\end{restatable}

Furthermore, in the Appendix, we show that the output values $v_t(\bm{w})$ is well-behaved in its domain of definition:
\begin{restatable}{thm}{thmFrameworkRelaxationMC}
\label{thm:Framework:Relaxation:MC}
The value of the output tuples, $v_t(\bm{w})$, varies monotonically with the rule weights $\bm{w}$, and is continuous
in the region $0 < w_r < 1$.
\end{restatable}

We could conceivably have chosen a different semiring in our definitions in Equations~\ref{eq:Framework:Relaxation:Tree}
and~\ref{eq:Framework:Relaxation:Tuple}. One alternative would be to choose a space of events, corresponding to the
inclusion of individual rules, and choosing the union and intersection of events as the semiring operations. This choice
would make the system coincide with ProbLog \cite{ProbLog}. However, the \#P-completeness of inference in probabilistic logics would
make the learning process computationally expensive. Other possibilities, such as the arithmetic semiring $(\mathbb{R},
+, \times, 0, 1)$, would lead to unbounded values for output tuples in the presence of infinitely many derivation
trees.

%%%%%%%%%%%%%%%%%%%%%%%%%%%%%%%%%%%%%%%%%%%%%%%%%%%%%%%%%%%%%%%%%%%%%%%%%%%%%%%%%%%%%%%%%%%%%%%%%%%%%%%%%%%%%%%%%%%%%%%%
% 2. Evaluation, Provenance, and Automatic Differentiation

\subsection{Evaluating and Automatically Differentiating \Difflog{} Programs}
\label{sub:Framework:Evaluation}

Because the set of derivation trees for an individual tuple $t$ may be infinite, note that Equation~%
\ref{eq:Framework:Relaxation:Tuple} is merely \emph{definitional}, and does not prescribe an algorithm to \emph{compute}
$v_t(\bm{w})$. Furthermore, numerical optimization requires the ability to automatically differentiate these values,
i.e., to compute $\nabla_{\bm{w}} v_t$.

The key to automatic differentiation is tracking the \emph{provenance} of each output tuple~\cite{SemiringProvenance}.
Pick an output tuple $t$, and let $\tau$ be its derivation tree with greatest value. For the purposes of this paper, the
provenance of $t$ is a map, $l_t = \{ r \mapsto \#r \text{ in } \tau \mid r \in R \}$, which maps each rule $r$ to the
number of times it appears in $\tau$. Given the provenance $l_t$ of a tuple, the derivative of
$v_t(\bm{w})$ can be readily computed: $d v_t(\bm{w}) / d w_r = l_t(r) v_t(\bm{w}) / w_r$.

\begin{algorithm}[t]\small
\caption{$\Evaluate(R, \bm{w}, I)$, where $R$ is a set of rules, $\bm{w}$ is an assignment of weight to each rule in
  $R$, and $I$ is a set of input tuples.}
\label{alg:Framework:Evaluation}
\begin{enumerate}[leftmargin=*]
\item Initialize the set of tuples in each relation, $F_P \coloneqq \emptyset$, their valuations, $u(t)
  \coloneqq 0$, and their provenance $l(t) = \{ r \mapsto \infty \mid r \in R \}$.
\item For each input relation $P$, update $F_P \coloneqq I_P$, and for each $t \in I_P$, update
   $u(t) \coloneqq 1$ and $l(t) = \{ r \mapsto 0 \mid r \in R \}$.
\item Until $(F, \bm{u})$ reach fixpoint,
  \begin{enumerate}[leftmargin=*,itemsep=0pt]
  \item Compute the immediate consequence of each rule, $r$, ``$P_h(\bm{u}_h) \leads P_1(\bm{u}_1), P_2(\bm{u}_2),
    \dots, P_k(\bm{u}_k)$'':
    \[
        F'_{P_h} = \pi_{\bm{u}_h}(F_{P_1}(\bm{u}_1) \bowtie F_{P_2}(\bm{u}_2) \bowtie \cdots \bowtie F_{P_k}(\bm{u}_k)).
    \]
    Furthermore, for each tuple $t \in F'_{P_h}$, determine all sets of antecedent tuples, $A_g(t) =
    \{ P_1(\bm{v}_1), P_2(\bm{v}_2), \dots, P_k(\bm{v}_k) \}$, which result in its production.
  \item Update $F_{P_h} \coloneqq F_{P_h} \union F'_{P_h}$.
  \item For each tuple $t \in F'_{P_h}$ and each $A_g(t)$:
    \begin{inparaenum}[(\itshape i\upshape)]
    \item compute $u'_t = w_r \prod_{i = 1}^k u(P_i(\bm{v}_i))$, and
    \item if $u(t) < u'_t$, update:
    	\vspace{-0.1in}
	\begin{align*}
        u(t)  \coloneqq u'_t, \text{ and~ }
        l(t)  \coloneqq \{ r \mapsto 1 \} + \sum_{i = 1}^k l(P_i(\bm{v}_i)),
        \end{align*}
      where addition of provenance values corresponds to the element-wise sum.
    \end{inparaenum}
  \end{enumerate}
\item Return $(F, \bm{u}, \bm{l})$.
\end{enumerate}
\end{algorithm}

In Algorithm~\ref{alg:Framework:Evaluation}, we present an algorithm to compute the output values $v_t(\bm{w})$ and
provenance $l_t$, given $R$, $\bm{w}$, and the input tuples $I$. The algorithm is essentially an instrumented
version of the ``naive'' Datalog evaluator~\cite{Alice}. We outline the proof of the following correctness and
complexity claims in the Appendix.
\vspace{-0.07in}
\begin{restatable}{thm}{thmFrameworkEvaluation}
\label{thm:Framework:Evaluation}
Fix a set of input relations $\mathcal{I}$, output relations $\mathcal{O}$, and candidate rules $R$. Let $\Evaluate(R,
\bm{w}, I) = (F, \bm{u}, \bm{l})$. Then:
\begin{inparaenum}[(\itshape a\upshape)]
\item $F = R(I)$, and
\item $\bm{u}(t) = v_t(\bm{w})$.
\end{inparaenum}
Furthermore, $\Evaluate(R, \bm{w}, I)$ returns in time $poly(|I|)$.
\end{restatable}

\section{Formulating the Optimization Problem}
\label{sec:Optimization}

We formulate the \Difflog{} synthesis problem as finding the value of the rule weights $\bm{w}$ which minimizes the
difference between the output values of tuples, $v_t(\bm{w})$, and their expected values, $1$ if $t \in O_+$, and $0$ if
$t \in O_-$. Specifically, we seek to minimize the L2 loss,
\begin{alignat}{1}
  L(\bm{w}) & = \sum_{t \in O_+} (1 - v_t(\bm{w}))^2 + \sum_{t \in O_-} v_t(\bm{w})^2.
\end{alignat}
At the optimum point, Theorem~\ref{thm:Framework:Relaxation:Refinement} enables the recovery of a classical Datalog
program from the optimum value $\bm{w}^*$.

\paragraph{Hybrid optimization procedure.}
In program synthesis, the goal is often to ensure exact compatibility with the provided positive and negative examples.
We therefore seek zeros of the loss function $L(\bm{w})$, and solve for this using Newton's root-finding algorithm:
$\bm{w}^{(i + 1)} = \bm{w}^{(i)} - L(\bm{w}) \nabla_{\bm{w}} L(\bm{w}) / \| \nabla_{\bm{w}} L(\bm{w}) \|^2$. To escape from local
minima and points of slow convergence, we periodically intersperse iterations of the MCMC sampling, specifically simulated annealing. We describe the parameters of the optimization algorithm in the Appendix.

\paragraph{Separation-guided search termination.}
After computing each subsequent $\bm{w}^{(i)}$, we examine the provenance values for each output tuple to determine
whether the current position can directly lead to a solution to the rule selection problem. In particular, we compute
the sets of desirable---$R_+ = \{ r \in l(t) \mid t \in O_+ \}$---and undesirable rules---$R_- = \{ r \in l(t) \mid t
\in O_- \}$, and check whether $R_+ \intersection R_- = \emptyset$. If these sets are separate, then we examine
the candidate solution $R_+$, and return if it satisfies the output specification.

\section{Empirical Evaluation}
\label{sec:Experiments}

Our experiments address the following aspects of \Difflog{}:
\begin{enumerate}[leftmargin=*,itemsep=0pt]
\item effectiveness at synthesizing Datalog programs and comparison to the state-of-the-art tool \alps~\cite{ALPS},
which already outperforms existing ILP tools~\cite{Zaatar,Metagol} and supports relations with arbitrary arity, sophisticated joins, and predicate invention;
\item the benefit of employing MCMC search compared to a purely gradient-based method; and
\item scaling with number of training labels and rule templates.
% \item How sensitive is \Difflog{} to choices of loss function, optimization method, MCMC-vs.-gradient descent, etc.
% \KH{TODO}
\end{enumerate}

We evaluated \Difflog{} on a suite of 34 benchmark problems~\cite{ALPS}. This collection draws benchmarks from three
different application domains:
\begin{inparaenum}[(\itshape a\upshape)]
\item knowledge discovery,
\item program analysis, and
\item relational queries.
\end{inparaenum}
The characteristics of the benchmarks are shown in Table~\ref{tbl:benchmarks} of the Appendix. These benchmarks involve
up to 10 target rules, which could be recursive and involve relations with arity up to 6. The implementation of
\Difflog{} comprises 4K lines of Scala code. We use Newton's root-finding method for continuous optimization and apply
MCMC-based random sampling every 30 iterations. All experiments were conducted on Linux machines with Intel Xeon 3GHz
processors and 64GB memory.

\subsection{Effectiveness}
\label{sec:experiment:effectiveness}
We first evaluate the effectiveness of \Difflog{} and compare it with \alps{}.
The running time and solution of \Difflog{} depends on the random choice of initial weights.
\Difflog{} exploits this characteristic by running multiple synthesis processes for each problem in parallel.
The solution is returned once one of the parallel processes successfully synthesizes a correct Datalog program.
We populated 32 processes in parallel and measured the running time until the first solution was found.
The timeout is set to 1 hour for each problem.

\begin{table}[t!]
\caption{Characteristics of benchmarks and performance of \Difflog{} compared to \alps.
\textbf{Rel} shows the number of relations.
\textbf{Rule} represents the number of expected (\textbf{Exp}) and candidate rules (\textbf{Cnd}).
\textbf{Tuple} shows the number of input and output tuples.
\textbf{Iter} and \textbf{Smpl} report the number of iterations and MCMC samplings. \textbf{Time} shows the running time of \Difflog{} and \alps{} in seconds.}
\vspace{-0.1in}
\centering
\small
\resizebox{\columnwidth}{!}{
\begin{tabular}{@{}l@{\ }r@{\ \ }r@{\ \ }r@{\ \ }r@{\ \ }r@{\ \ }r@{\ \ }r@{\ \ }r@{\ \ }r@{}}
\toprule
\multirow{2}{*}{\textbf{Benchmark}} &
\multirow{2}{*}{\textbf{Rel}} &
\multicolumn{2}{c}{\textbf{Rule}} &
\multicolumn{2}{c}{\textbf{Tuple}} &
\multicolumn{3}{c}{\Difflog{}} &
\multicolumn{1}{@{}r@{}}{\alps{}} \tabularnewline
\cmidrule(l{0pt}r{4pt}){3-4}
\cmidrule(l{0pt}r{4pt}){5-6}
\cmidrule(l{0pt}r{4pt}){7-9}\cmidrule(l{0pt}r{1pt}){10-10}
& & \textbf{Exp} & \textbf{Cnd} &  \textbf{In} & \textbf{Out}  & \textbf{Iter} & \textbf{Smpl} & \textbf{Time} & \textbf{Time} \tabularnewline
\midrule
\texttt{inflamation}	&	7	&	2	&	134	&	640	&	49	&	1	&	0	&	\textbf{1}	&	2	\tabularnewline	
\texttt{abduce}	&	4	&	3	&	80	&	12	&	20	&	1	&	0	&	$<$ \textbf{1}	&	2	\tabularnewline	
\texttt{animals}	&	13	&	4	&	336	&	50	&	64	&	1	&	0	&	\textbf{1}	&	40	\tabularnewline	
\texttt{ancestor}	&	4	&	4	&	80	&	8	&	27	&	1	&	0	&	$<$ \textbf{1}	&	14	\tabularnewline	
\texttt{buildWall}	&	5	&	4	&	472	&	30	&	4	&	5	&	1	&	\textbf{7}	&	67	\tabularnewline	
\texttt{samegen}	&	3	&	3	&	188	&	7	&	22	&	1	&	0	&	\textbf{2}	&	12	\tabularnewline	
\texttt{scc}	&	3	&	3	&	384	&	9	&	68	&	6	&	1	&	\textbf{28}	&	60	\tabularnewline	\midrule
\texttt{polysite}	&	6	&	3	&	552	&	97	&	27	&	17	&	1	&	\textbf{27}	&	84	\tabularnewline	
\texttt{downcast}	&	9	&	4	&	1,267	&	89	&	175	&	5	&	1	&	\textbf{30}	&	1,646	\tabularnewline	
\texttt{rv-check}	&	5	&	5	&	335	&	74	&	2	&	1,205	&	41	&	\textbf{22}	&	195	\tabularnewline	
\texttt{andersen}	&	5	&	4	&	175	&	7	&	7	&	1	&	0	&	\textbf{4}	&	27	\tabularnewline	
\texttt{1-call-site}	&	9	&	4	&	173	&	28	&	16	&	4	&	1	&	\textbf{4}	&	106	\tabularnewline	
\texttt{2-call-site}	&	9	&	4	&	122	&	30	&	15	&	25	&	1	&	\textbf{53}	&	676	\tabularnewline	
\texttt{1-object}	&	11	&	4	&	46	&	40	&	13	&	3	&	1	&	\textbf{3}	&	345	\tabularnewline	
\texttt{1-type}	&	12	&	4	&	70	&	48	&	22	&	3	&	1	&	\textbf{4}	&	13	\tabularnewline	
\texttt{escape}	&	10	&	6	&	140	&	13	&	19	&	2	&	1	&	\textbf{1}	&	5	\tabularnewline	
\texttt{modref}	&	13	&	10	&	129	&	18	&	34	&	1	&	0	&	\textbf{1}	&	2,836	\tabularnewline	\midrule
\texttt{sql-10}	&	3	&	2	&	734	&	10	&	2	&	7	&	1	&	\textbf{11}	&	41	\tabularnewline	
\texttt{sql-14}	&	4	&	3	&	23	&	11	&	6	&	1	&	0	&	$<$ \textbf{1}	&	54	\tabularnewline	
\texttt{sql-15}	&	4	&	2	&	186	&	50	&	7	&	902	&	31	&	875	&	\textbf{11}	\tabularnewline	
\bottomrule
\end{tabular}}
\vspace{-0.1in}
\label{tbl:performance}
\end{table}

Table~\ref{tbl:performance} shows the running of \Difflog{} and \alps.
We excluded 14 out of 34 benchmarks that both \Difflog{} and \alps{} solve within a second (13 benchmarks) or run out of time (1~benchmark).
\Difflog{} outperforms \alps{} on 19 of the remaining 20 benchmarks.
%\alps{} runs faster than \Difflog{} for 3 benchmarks and only one of them shows substantial difference.
In particular, \Difflog{} is orders of magnitude faster than \alps{} on most of the program analysis benchmarks.
Meanwhile, the continuous optimization may not be efficient when the problem has many local minimas
and the space is not convex. 
For example, \texttt{sql-15} has a lot of sub-optimal solutions that generate not only all positive output tuples but also some negative ones.

\begin{figure}[t]
\center
\includegraphics[width=0.78\linewidth]{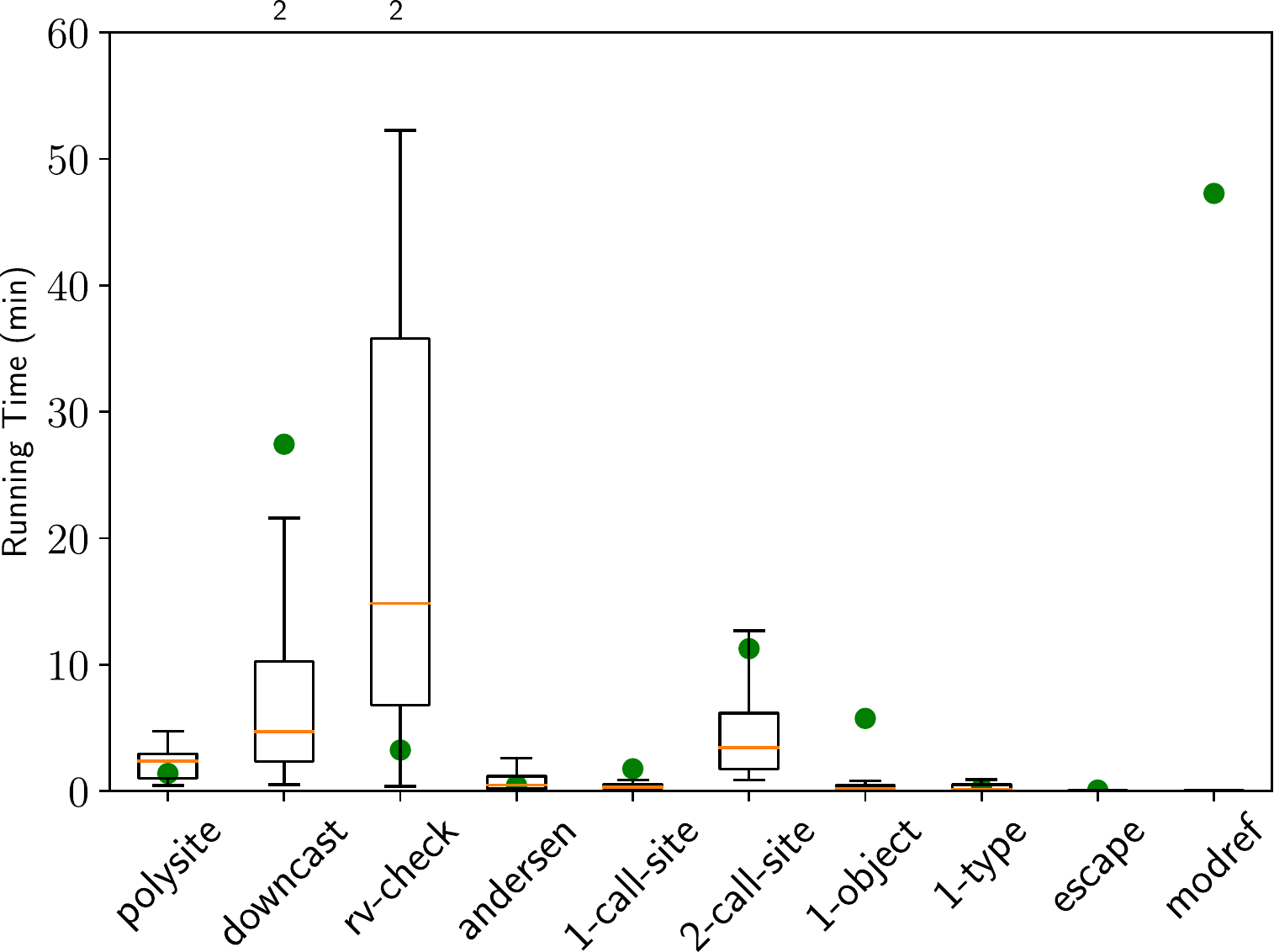}
\vspace{-0.1in}
\caption{Distribution of \Difflog{}'s running time from 32 parallel runs. The numbers on top represents the number of timeouts. Green circles represent the running time of \alps.}
\label{fig:distribution}
\end{figure}

Figure~\ref{fig:distribution} depicts the distribution of running time on the program analysis benchmarks.\footnote{We provide results for the other domains in the Appendix.}
The results show that \Difflog{} is always able to find solutions for all the benchmarks except for 2 timeouts for \texttt{downcast} and \texttt{rv-check} respectively.
Also note that even the median running time of \Difflog{} is smaller than the running time of \alps{} for 6 out of 10 benchmarks.

\subsection{Impact of MCMC-based sampling}
\begin{table}[t!]\small\center
\caption{Effectiveness of MCMC  sampling in terms of the best running time (\textbf{Time}) and the number of timeouts (\textbf{T/O}).}
\vspace{-0.1in}
\resizebox{0.47\textwidth}{!}{
\begin{tabular}{l@{\quad}rr@{\quad}rrrr}
\toprule
\multirow{2}{*}{\textbf{Benchmark}} &
\multicolumn{2}{c@{\quad}}{\textbf{Hybrid}} &
\multicolumn{2}{c}{\textbf{Newton}}  &
\multicolumn{2}{c}{\textbf{MCMC}}
\tabularnewline
\cmidrule(l{0pt}r{10pt}){2-3}
\cmidrule(l{0pt}r{6pt}){4-5}
\cmidrule(l{6pt}r{6pt}){6-7}
& \textbf{Time} & \textbf{T/O}
& \textbf{Time} & \textbf{T/O}
& \textbf{Time} & \textbf{T/O}  \tabularnewline\midrule
\texttt{polysite}	&	27	&	0	&	10	&	0	&	12	&	0	\tabularnewline	
\texttt{downcast}	&	30	&	2	&	16	&	9	&	70	&	7	\tabularnewline	
\texttt{rv-check}	&	22	&	2	&	N/A	&	32	&	N/A	&	32	\tabularnewline	
\texttt{andersen}	&	4	&	0	&	3	&	10	&	4	&	9	\tabularnewline	
\texttt{1-call-site}	&	4	&	0	&	8	&	1	&	N/A	&	32	\tabularnewline	
\texttt{2-call-site}	&	53	&	0	&	27	&	17	&	42	&	9	\tabularnewline	
\texttt{1-object}	&	3	&	0	&	3	&	17	&	N/A	&	32	\tabularnewline	
\texttt{1-type}	&	4	&	0	&	3	&	18	&	N/A	&	32	\tabularnewline	
\texttt{escape}	&	1	&	0	&	1	&	17	&	N/A	&	32	\tabularnewline	
\texttt{modref}	&	1	&	0	&	1	&	4	&	N/A	&	32	\tabularnewline	\midrule
\textbf{Total}	&		&	4	&		&	125	&		&	217	\tabularnewline	
\bottomrule
\end{tabular}
}
\label{tbl:mcmc}
\vspace{-0.1in}
\end{table}

We next evaluate the impact of our MCMC-based sampling by comparing the performance of three variants of \Difflog{}:
\begin{inparaenum}[\itshape a)\upshape]
\item a version that uses both Newton's method and the MCMC-based technique (\textbf{Hybrid}), which is the same as in Section~\ref{sec:experiment:effectiveness},
\item a version that uses only Newton's method (\textbf{Newton}), and
\item a version that uses only the MCMC-based technique (\textbf{MCMC}).
\end{inparaenum}
Table~\ref{tbl:mcmc} shows the running time of the best run and the number of timeouts among 32 parallel runs
for these three variants.
The table shows that our hybrid approach strikes a good balance between exploitation and exploration. 
%in complex search spaces of Datalog synthesis problems.
In many cases, \textbf{Newton} gets stuck in local minima; for example, it cannot find any solution for \texttt{rv-check} within one hour.
\textbf{MCMC} cannot find any solution for 6 out of 10 benchmarks.
Overall, \textbf{Hybrid} outperforms both \textbf{Newton} and \textbf{MCMC} by reporting 31$\times$  and 54$\times$ less number of timeouts, respectively.

\subsection{Scalability}

\noindent
We next evaluate the scalability of \Difflog{}, which is essentially affected by two factors: 
the number of templates and the size of training data. 
Our general observation is that increasing either of these does not significantly increase the
effective running time (i.e., the best of 32 parallel runs).

Figure~\ref{fig:vary_templates} shows how running time increases with the number of 
templates.\footnote{We ensure that a smaller set is always a subset of a larger one.}
As shown in Figure~\ref{fig:2call}, the running time distribution for \twocallsite{} tends to have 
larger variance when the number of templates increases, but the best running time (out of 32 i.i.d samples)
only increases modestly.
The running time distribution for \downcast{}, shown in Figure~\ref{fig:downcast}, has a similar trend
except that smaller number of templates does not always lead to smaller variance or faster running time.
For instance, the distribution in the setting with 180 templates has larger variance and
median than distributions in the subsequent settings with larger number of templates. 
This indicates that the actual combination of templates also matters. 

\begin{figure}
  \vspace{-0.2in}
  \centering 
  \begin{subfigure}[b]{0.49\linewidth}
	\includegraphics[width=0.98\linewidth]{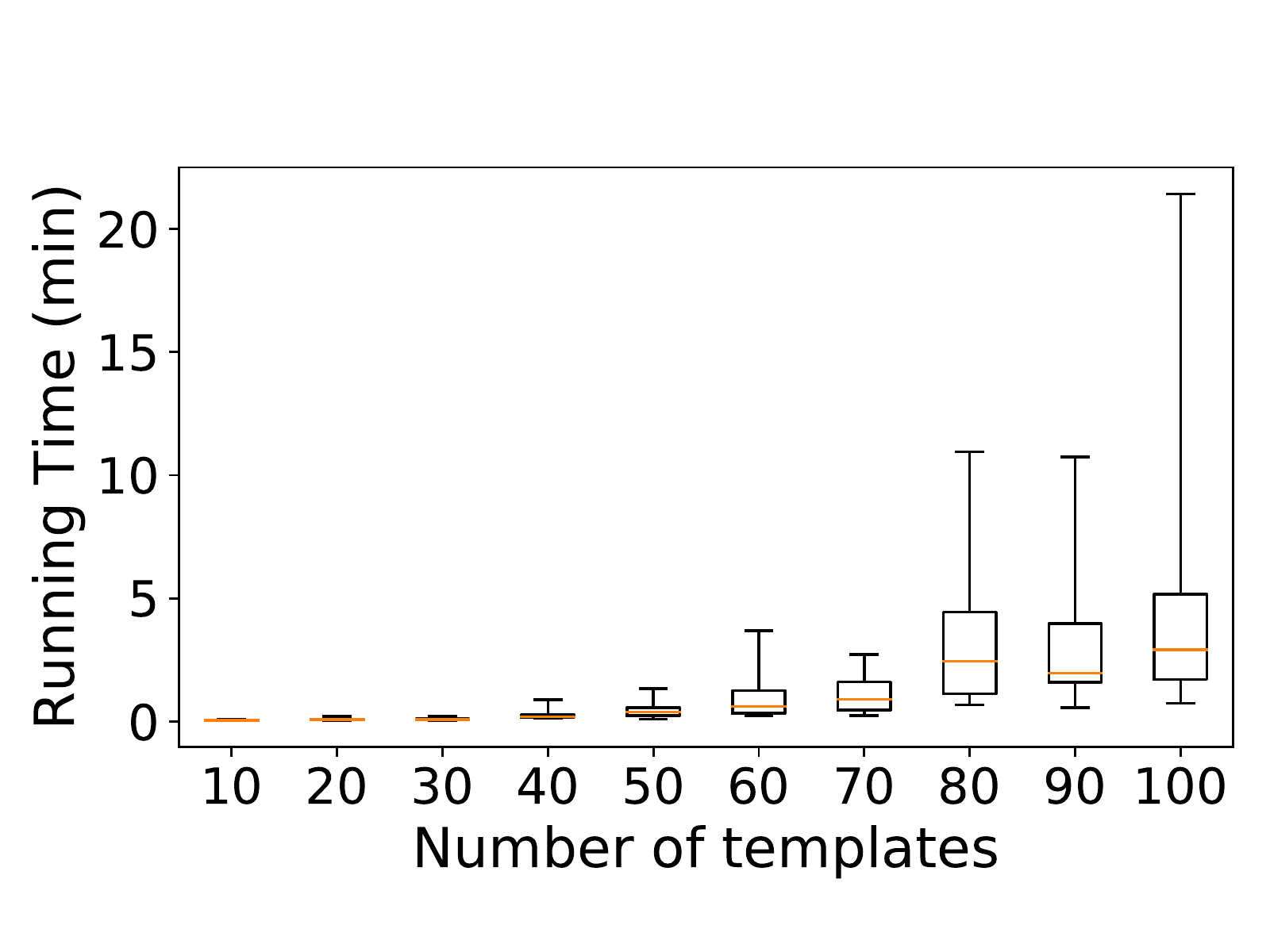}
	\vspace{-5pt}
  	\caption{\twocallsite}
  	\label{fig:2call}
  \end{subfigure}
  \begin{subfigure}[b]{0.49\linewidth}
	\includegraphics[width=0.98\linewidth]{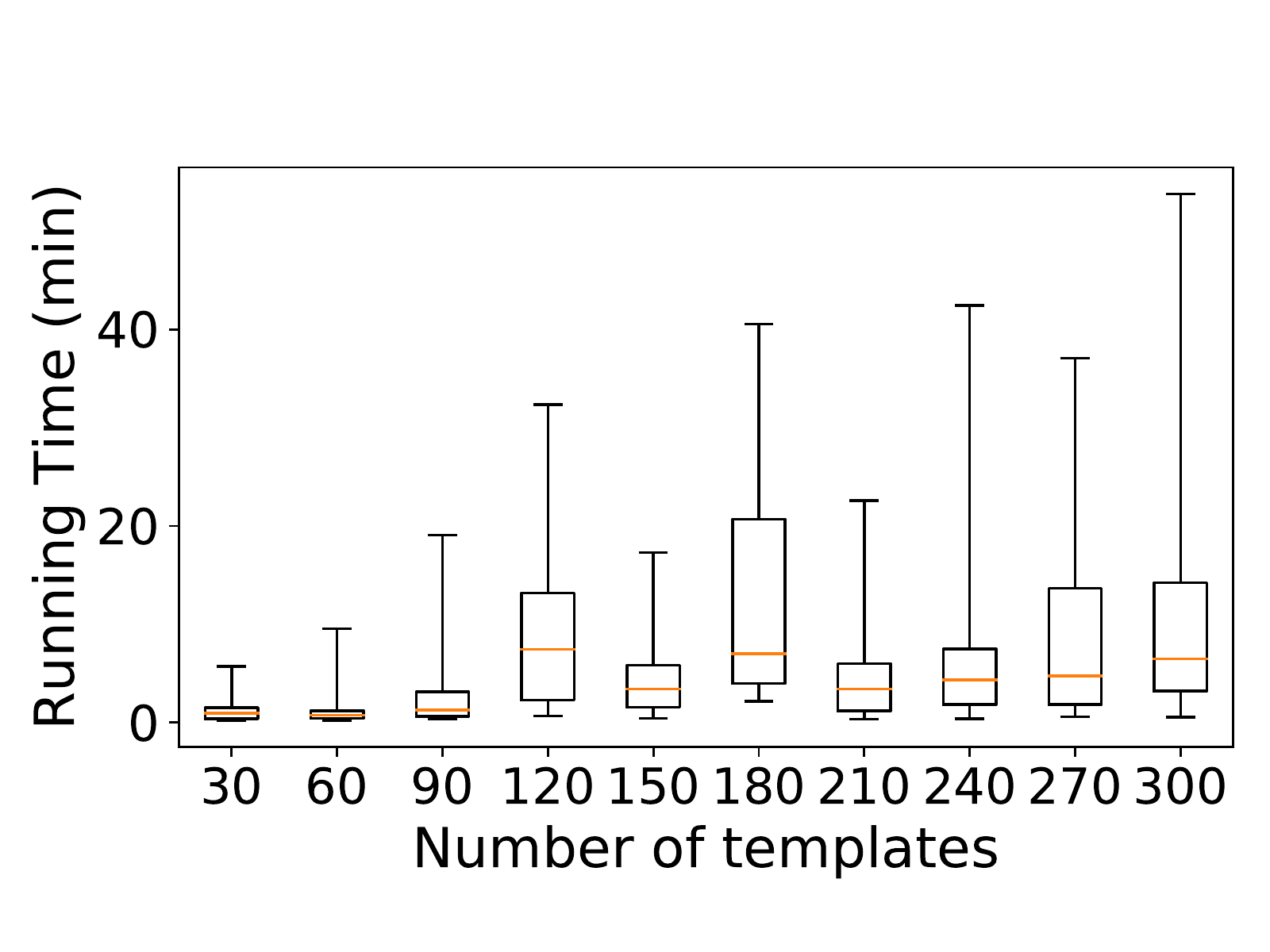}
	\vspace{-5pt}
  	\caption{\downcast}
  	\label{fig:downcast}
  \end{subfigure}
  \vspace{-5pt}
  \caption{Running time distributions for \twocallsite{} and \downcast{} with different number of templates.}
  \label{fig:vary_templates}
\end{figure}

%Figure~\ref{fig:vary_tuples} shows the trend of running time 

\begin{figure}
  \vspace{-0.2in}
  \centering
%  \begin{subfigure}[b]{0.33\linewidth}
%	\includegraphics[width=0.98\linewidth]{figs/andersen_eval_avg}
%	\vspace{-5pt}
%  	\caption{}
%  	\label{fig:andersen_eval}
%  \end{subfigure}
  \begin{subfigure}[b]{0.49\linewidth}
	\includegraphics[width=0.98\linewidth]{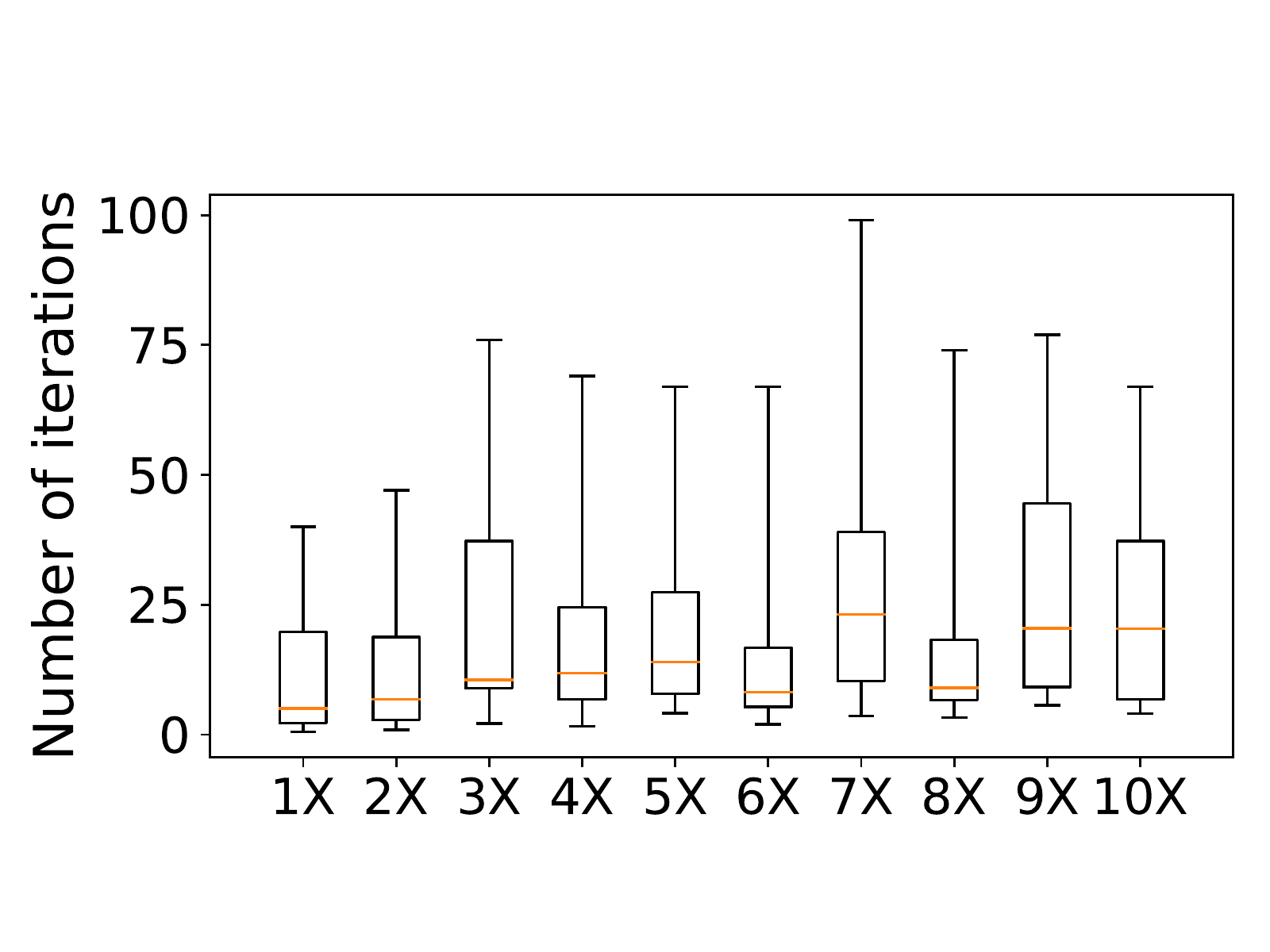}
	\vspace{-15pt}
  	\caption{}
  	\label{fig:andersen_iter}
  \end{subfigure}
  \begin{subfigure}[b]{0.49\linewidth}
	\includegraphics[width=0.98\linewidth]{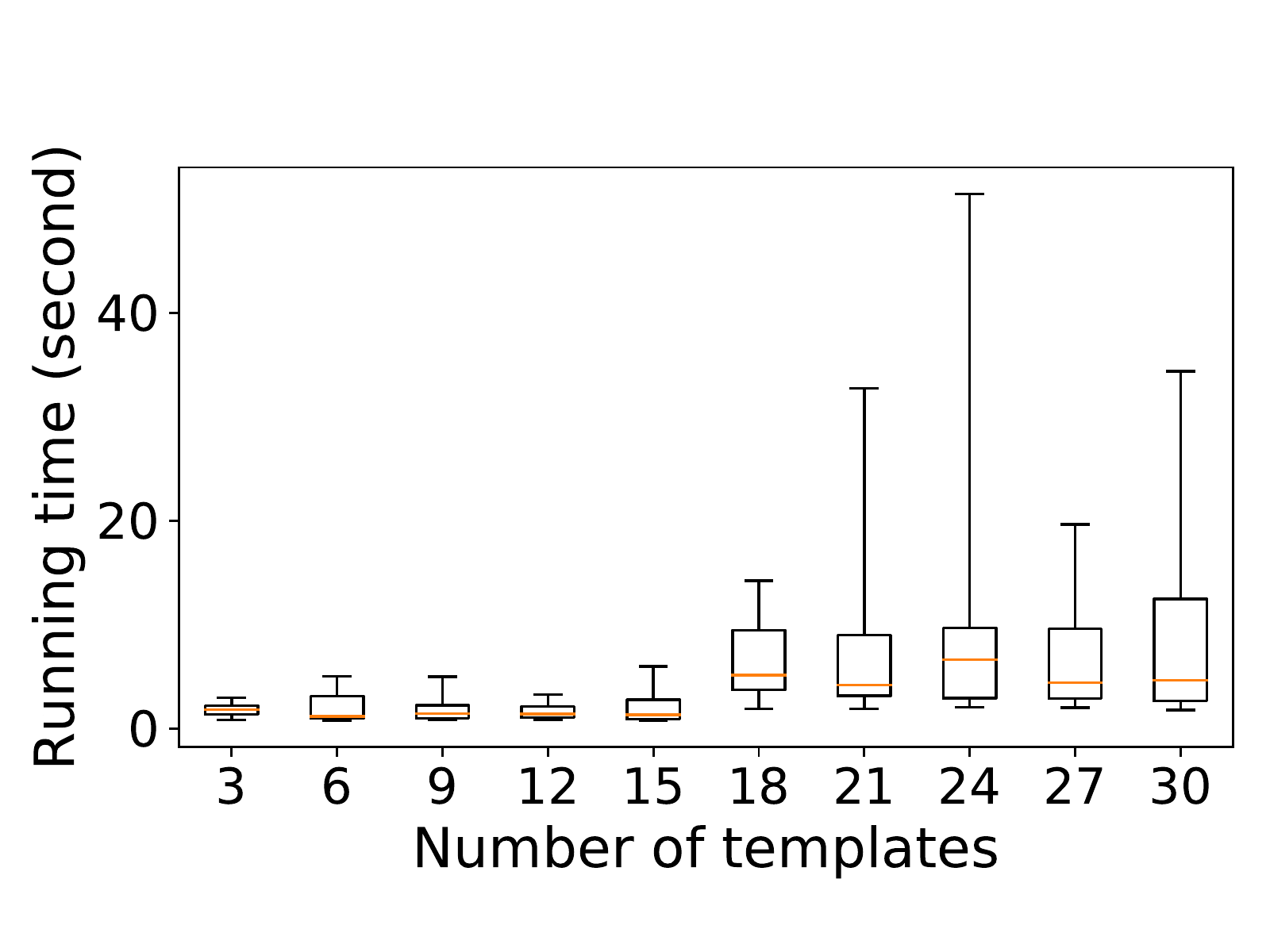}
	\vspace{-15pt}
  	\caption{}
  	\label{fig:andersen_dist}
  \end{subfigure}
  	\vspace{-8pt}
	\caption{Performance of \Difflog{} on \andersen{} with different sizes of data: (a) the distribution of number of iterations, (b) the distribution of running time.}
	\label{fig:vary_tuples}
\vspace{-0.1in}
\end{figure}

The size of training data is another important factor affecting the performance of \Difflog{}. 
%Figure~\ref{fig:vary_tuples} shows the results for \andersen{} with increasing sizes of training data.
Figure~\ref{fig:andersen_iter} shows the distribution of the number of iterations for \andersen{} with different sizes of training data.
According to the results, the size of training data does not necessarily affect the number of iterations of \Difflog{}.
Meanwhile, Figure~\ref{fig:andersen_dist} shows that the end-to-end running time increases with more training data.
This is mainly because more training data impose more cost on the \Difflog{} evaluator.
However, the statistics shows that the running time increases linearly with the size of  data.

%On one hand, more training data means more burden for the rule evaluator: as shown in Figure~\ref{fig:andersen_eval}, the average evaluation time increases proportionally to the size of data.
%On the other hand, 
%More training data indirectly reduces the number of iterations to converge.
%More training data means less opportunities of overfitting, which potentially avoids many local minimal configurations of weights. 

%Note that the number of iterations slightly reduces with more training data.
%Figure~\ref{fig:andersen_dist} shows the combined effects, that is, the end-to-end running time distribution.
%Note that the running time sometimes reduces with more training data (e.g., 6$\times$ vs. 3$\times$/4$\times$/5$\times$).

\section{Conclusion}
\label{sec:Conclusion}

We presented a technique to synthesize Datalog programs by numerical optimization. The central idea was to formulate the
problem as an instance of rule selection, and then relax classical Datalog to a refinement named \Difflog. In a
comprehensive set of experiments, we show that by learning a \Difflog{} program and then recovering a classical Datalog
program, we can achieve significant speedups over the state-of-the-art Datalog synthesis systems. In future, we plan to
extend the approach to other synthesis problems such as SyGuS and to applications in differentiable programming.

\bibliographystyle{named}
\bibliography{references}
\clearpage

\appendix
\section{Proof of Theorems~\ref{thm:Problem:Complexity},~\ref{thm:Framework:Relaxation:MC}
  and~\ref{thm:Framework:Evaluation}}
\label{app:Problem}

%%%%%%%%%%%%%%%%%%%%%%%%%%%%%%%%%%%%%%%%%%%%%%%%%%%%%%%%%%%%%%%%%%%%%%%%%%%%%%%%%%%%%%%%%%%%%%%%%%%%%%%%%%%%%%%%%%%%%%%%
% 1. Complexity of Rule Selection

\thmProblemComplexity*
\begin{proof}
Consider a 3-CNF formula $\varphi$ over a set $V$ of variables:
\[
  \varphi = (l_{11} \lor l_{12} \lor l_{13}) \land
            (l_{21} \lor l_{22} \lor l_{23}) \land
            \dots \land
            (l_{k1} \lor l_{k2} \lor l_{k3}),
\]
be the given 3-CNF formula, where each literal $l_{ij}$ appearing in clause $c_i$ is either a variable, $v_{ij} \in V$,
or its negation, $\lnot v_{ij}$. Assume that there are no trivial clauses in $\varphi$, which simultaneously contain
both a variable and its negation. We will now encode its satisfiability as an instance of the rule selection problem.

\begin{enumerate}
\item For each variable $v \in V$, define the input relations:
  \begin{alignat}{1}
    \sdrel{pos}_v & = \{ (c) \mid v \in c \}, \text{ and} \label{eq:Problem:Complexity:Pos} \\
    \sdrel{neg}_v & = \{ (c) \mid \lnot v \in c \}, \label{eq:Problem:Complexity:Neg}
  \end{alignat}
  consisting of all one-place tuples $\sdrel{pos}_v(c)$ and $\sdrel{neg}_v(c)$ indicating whether the variable $v$
  occurs positively or negatively in the clause $c$.

\item Also, for each variable $v$, define the input relation $\sdrel{var}_v$ which is inhabited by a single tuple
  $\sdrel{var}_v(v)$:
  \begin{alignat}{1}
    \sdrel{var}_v & = \{ (v) \}. \label{eq:Problem:Complexity:Var}
  \end{alignat}

\item The idea is to set up the candidate rules so that subsets of chosen rules correspond to assignments of true /
  false values to the variables of $\varphi$. Let $C_2(c, v)$ be an output relation: we are setting up the problem so
  that if the tuple $C_2(c, v)$ is derivable in the synthesized solution, then there is a satisfying assignment of
  $\varphi$ where clause $c$ is satisfied due to the assignment to variable $v$.

\item For each variable $v$, create a pair of candidate rules $r_v$ and $r_{\lnot v}$ as follows:
  \begin{alignat*}{1}
    r_v & = \text{``} C_2(c, v') \leads \sdrel{pos}_v(c), \sdrel{var}_v(v') \text{''}, \text{ and} \\
    r_{\lnot v} & = \text{``} C_2(c, v') \leads \sdrel{neg}_v(c), \sdrel{var}_v(v') \text{''}.
  \end{alignat*}
  Selecting the rule $r_v$ corresponds to assigning the value true to the corresponding variable $v$, and selecting the
  rule $r_{\lnot v}$ corresponds to assigning it the value false.

\item To prevent the simultaneous choice of rules $r_v$ and $r_{\lnot v}$, we set up the three-place input relation
  $\sdrel{conflict}(c, c', v)$, which indicates that the reason for the simultaneous satisfaction of clauses $c$ and
  $c'$ cannot be a contradictory variable $v$:
  \begin{alignat}{1}
    \sdrel{conflict} & = \{ (c, c', v) \mid v \in c \text{ and } \lnot v \in c' \} \union \{ (a, a, a) \},
    \label{eq:Problem:Complexity:Conflict}
  \end{alignat}
  where $a$ is some new constant not seen before. We will motivate its necessity while defining the canary output
  relation $\sdrel{error}$ next.

\item We detect the simultaneous selection of a pair of rules $r_v$ and $r_{\lnot v}$ using the rule $r_e$:
  \begin{alignat*}{1}
    r_e = \text{``} \sdrel{error}(c, c', v) \leads C_2(c, v), C_2(c', v), \sdrel{conflict}(c, c', v) \text{''}
  \end{alignat*}
  Here $\sdrel{error}$ is a three-place output relation indicating the selection of an inconsistent assignment. We would
  like to force the synthesizer to choose the error-detecting rule $r_e$. The selection of the rule $r_e$, the presence
  of the input tuple $\sdrel{conflict}(a, a, a)$, and the selection of the rule $r_a$:
  \[
    r_a = \text{``} C_2(x, x) \leads \sdrel{conflict}(x, x, x) \text{''}
  \]
  is the only way to produce the output tuple $\sdrel{error}(a, a, a)$, which we will mark as desired.

\item The output tuple $C_2(c, v)$ indicates the satisfaction of the clause $c$ because of the assignment to variable
  $v$. We use the presence of such tuples to mark the clause $c$ itself as being satisfied: let $C_1(c)$ be a one-place
  output relation, and include the rule:
  \[
    r_c = \text{``} C_1(c) \leads C_2(c, v) \text{''}.
  \]

\item In summary, let the rule selection problem $P_\varphi = (\mathcal{I}, \mathcal{O}, I, O_+, O_-, R)$ be defined as
  follows:
  \begin{enumerate}
  \item $\mathcal{I} = \{ \sdrel{var}_v, \sdrel{pos}_v, \sdrel{neg}_v \mid v \in V \} \union \{ \sdrel{conflict} \}$.
  \item $\mathcal{O} = \{ C_2, C_1, \sdrel{error} \}$.
  \item Define the set of input tuples, $I$, using equations~\ref{eq:Problem:Complexity:Pos},
    \ref{eq:Problem:Complexity:Neg}, \ref{eq:Problem:Complexity:Var}, and~\ref{eq:Problem:Complexity:Conflict}.
  \item $O_+ = \{ C_1(c) \mid \text{clause } c \in \varphi \} \union \{ \sdrel{error}(a, a, a) \}$.
  \item $O_- = \{ \sdrel{error}(c, c', v) \mid \text{clauses } c, c' \allowbreak \text{ and variable } v \allowbreak
    \text{ occurring in } \varphi \}$.
  \item $R = \{ r_v, r_{\lnot v} \mid v \in V \} \union \{ r_e, r_a, r_c \}$.
  \end{enumerate}
\end{enumerate}
Given a 3-CNF formula $\varphi$, the corresponding instance $P_\varphi$ of the rule selection problem can be constructed
in polynomial time. Furthermore, it can be seen that, by construction, $P_\varphi$ admits a solution iff $\varphi$ is
satisfiable. It follows that the rule selection problem is NP-hard.
\end{proof}

%%%%%%%%%%%%%%%%%%%%%%%%%%%%%%%%%%%%%%%%%%%%%%%%%%%%%%%%%%%%%%%%%%%%%%%%%%%%%%%%%%%%%%%%%%%%%%%%%%%%%%%%%%%%%%%%%%%%%%%%
% 2. Continuity of Difflog Semantics

Next, we turn our attention to Theorem~\ref{thm:Framework:Relaxation:MC}. The first part of the claim follows
immediately from the definition in Equation~\ref{eq:Framework:Relaxation:Tuple}. We therefore focus on the second part:
Note that the proof of continuity does not immediately follow from Equation~\ref{eq:Framework:Relaxation:Tuple} because
the supremum of an infinite set of continuous functions need not itself be continuous. It instead depends on the
observation that there is a finite subset of dominating derivation trees whose values suffice to compute $v_t(\bm{w})$.

\thmFrameworkRelaxationMC*
\begin{proof}
Fix an assignment of rule weights $\bm{w}$. Next, focus on a specific output tuple $t$, and consider the set of all its
derivation trees $\tau$. Let $\sigma_\tau$ be a pre-order traversal over its nodes. For example, for the tree $\tau_1$
in Figure~\ref{sfig:Framework:Relaxation:Trees:1}, we obtain
$\sigma_{\tau_1} = \sdrel{samegen}(\sdconst{Will}, \sdconst{Ann}), \allowbreak
                   r_1(\sdconst{Will}, \sdconst{Ann}, \sdconst{Noah}), \allowbreak
                   \sdrel{parent}(\sdconst{Will}, \sdconst{Noah}), \allowbreak
                   \sdrel{parent}(\sdconst{Ann}, \sdconst{Noah})$.
% For $\tau_2$, we obtain $\sigma_{\tau_2} = \sdrel{samegen}(\sdconst{Will}, \sdconst{Ann}), \allowbreak
%                                            r_2(\sdconst{Will}, \sdconst{Noah}, \sdconst{Ann}, \sdconst{Noah}), \allowbreak
%                                            \sdrel{parent}(\sdconst{Will}, \sdconst{Noah}), \allowbreak
%                                            \sdrel{parent}(\sdconst{Ann}, \sdconst{Noah}), \allowbreak
%                                            \sdrel{samegen}(\sdconst{Noah}, \sdconst{Noah}), \allowbreak
%                                            r_1(\sdconst{Noah}, \sdconst{Noah}, \sdconst{Liam}), \allowbreak
%                                            \sdrel{parent}(\sdconst{Noah}, \sdconst{Liam}), \allowbreak
%                                            \sdrel{parent}(\sdconst{Noah}, \sdconst{Liam})$.
It can be shown that the set of all pre-order traversals, $\sigma_\tau$, over all derivation trees $\tau$ forms a
context-free grammar $L_t$.

We are interested in trees $\tau$ with high values $v_\tau(\bm{w})$, where the value of a tree depends only on the
number of occurrences of each rule $r$. It therefore follows that the weight $v_\tau(\bm{w})$ is completely specified
by the Parikh image, $\{ r \mapsto \#r \text{ in } \tau \}$, which counts the number of occurrences of each symbol in
each string of the language $L_t$. From Parikh's lemma, we conclude that this forms a semilinear set. Let
\[
  p(L_t) = \bigcup_{i = 1}^m (\bm{c}_{i0} + \sum_{j = 1}^n \bm{c}_{ij})
\]
be the Parikh image of $L_t$, and for each $i \in \{ 1, 2, \dots, m \}$, let $\tau_i$ be the derivation tree
corresponding to the rule count $\bm{c}_{i0}$. It follows that:
\[
  v_t(\bm{w}) = \sup_{\tau \text{ with conclusion } t} v_\tau(\bm{w}) = \max_{i = 1}^m v_{\tau_{i}}(\bm{w}).
\]
We have reduced the supremum over an infinite set of continuous functions to the maximum of a finite set of continuous
functions. It follows that $v_t(\bm{w})$ varies continuously with $\bm{w}$.
\end{proof}

%%%%%%%%%%%%%%%%%%%%%%%%%%%%%%%%%%%%%%%%%%%%%%%%%%%%%%%%%%%%%%%%%%%%%%%%%%%%%%%%%%%%%%%%%%%%%%%%%%%%%%%%%%%%%%%%%%%%%%%%
% 3. Proof of Evalation Complexity

Finally, we turn to the proof of Theorem~\ref{thm:Framework:Evaluation}.

\thmFrameworkEvaluation*
\begin{proof}
The first part of the following result follows from similar
arguments as the correctness of the classical algorithm.
We briefly describe the proof of the second claim. For each output tuple $t$, consider all of its derivation trees
$\tau_{hi}$ with maximal value, and identify the tree $\tau_t$ with shortest height among these. All first-level
sub-trees of $\tau_t$ must themselves possess the shortest-height-maximal-value property, so that their height is bounded
by the number of output tuples. Since the $(F, \bm{u}, \bm{l})$-loop in step~3 of Algorithm~%
\ref{alg:Framework:Evaluation} has to hit a fixpoint within as many iterations, and since each iteration runs in
polynomial time, the claim about running time follows.
\end{proof}

\section{Learning Details}
\label{app:Params}

We initialize $\bm{w}$ by uniformly sampling weights $w_r \in [0.25, 0.75]$. We apply MCMC sampling after every 30
iterations of Newton's root-finding method, and sample new weights as follows:
\[ X \sim U(0,1) \]
\[ w_{new} = \bigg \{
  \begin{tabular}{ll}
  $w_{old} \sqrt{2X} $ & if $X < 0.5$\\
  $ 1 - (1-w_{old}) \sqrt{2(1-X)}$ & otherwise.
  \end{tabular} \]
The temperature $T$ used in simulated annealing is as follows:
\[ T = \frac{1.0}{C * log(5 + \#iter)}\]
where C is initially 0.0001 and $\#iter$ is the number of iterations.
We accept the newly proposed sample with probability
\[ p_{acc} = \min(1, \pi_{new} / \pi_{curr}), \]
where $\pi_{curr} = \exp(-L_2(\bm{w}_{curr}) / T)$ and $\pi_{new} = \exp(-L_2(\bm{w}_{new}) / T)$.

\section{Benchmarks and Experimental Results}
\label{app:Benchmarks}
The characteristics of benchmarks are shown in Table~\ref{tbl:benchmarks}.
Figure~\ref{fig:appendix:distribution} shows that the distribution of running time for the remaining benchmarks.
\begin{table}[t!]
\caption{Benchmarks characteristics.
\textbf{Rec} and \textbf{\#Rel} shows the programs that require recursive rules, and the number of relations.
\textbf{\#Rules} represents the number of expected (\textbf{Exp}) and candidate rules (\textbf{Cand}).
\textbf{\#Tuples} shows the number of input and output tuples.}
\centering
\small
\begin{tabular}{l@{\ \ }lrrrrrrr}
\toprule
 &
\multirow{2}{*}{\textbf{Benchmark}} &
\multirow{2}{*}{\textbf{Rec}}&
\multirow{2}{*}{\textbf{\#Rel}} &
\multicolumn{2}{c}{\textbf{\#Rules}} &
\multicolumn{2}{c}{\textbf{\#Tuples}} \tabularnewline
\cmidrule(l{2pt}r{2pt}){5-6}\cmidrule(l{2pt}r{2pt}){7-8}
& & & & \textbf{Exp} & \textbf{Cand} & \textbf{In} & \textbf{Out} \tabularnewline
\midrule
\multirow{8}{*}{\rotatebox[origin=c]{90}{\sf Knowelge Discovery}}	&	\texttt{inflamation}	&		&	7	&	2	&	134	&	640	&	49	\tabularnewline	
	&	\texttt{abduce}	&		&	4	&	3	&	80	&	12	&	20	\tabularnewline	
	&	\texttt{animals}	&		&	13	&	4	&	336	&	50	&	64	\tabularnewline	
	&	\texttt{ancestor}	&	\checkmark	&	4	&	4	&	80	&	8	&	27	\tabularnewline	
	&	\texttt{buildWall}	&	\checkmark	&	5	&	4	&	472	&	30	&	4	\tabularnewline	
	&	\texttt{samegen}	&	\checkmark	&	3	&	3	&	188	&	7	&	22	\tabularnewline	
	&	\texttt{path}	&	\checkmark	&	2	&	2	&	6	&	7	&	31	\tabularnewline	
	&	\texttt{scc}	&	\checkmark	&	3	&	3	&	384	&	9	&	68	\tabularnewline	\midrule
\multirow{11}{*}{\rotatebox[origin=c]{90}{\sf Program Analysis}}	&	\texttt{polysite}	&		&	6	&	3	&	552	&	97	&	27	\tabularnewline	
	&	\texttt{downcast}	&		&	9	&	4	&	1,267	&	89	&	175	\tabularnewline	
	&	\texttt{rv-check}	&		&	5	&	5	&	335	&	74	&	2	\tabularnewline	
	&	\texttt{andersen}	&	\checkmark	&	5	&	4	&	175	&	7	&	7	\tabularnewline	
	&	\texttt{1-call-site}	&	\checkmark	&	9	&	4	&	173	&	28	&	16	\tabularnewline	
	&	\texttt{2-call-site}	&	\checkmark	&	9	&	4	&	122	&	30	&	15	\tabularnewline	
	&	\texttt{1-object}	&	\checkmark	&	11	&	4	&	46	&	40	&	13	\tabularnewline	
	&	\texttt{1-type}	&	\checkmark	&	12	&	4	&	70	&	42	&	15	\tabularnewline	
	&	\texttt{1-obj-type}	&	\checkmark	&	13	&	5	&	12	&	48	&	22	\tabularnewline	
	&	\texttt{escape}	&	\checkmark	&	10	&	6	&	140	&	13	&	19	\tabularnewline	
	&	\texttt{modref}	&	\checkmark	&	13	&	10	&	129	&	18	&	34	\tabularnewline	\midrule
\multirow{15}{*}{\rotatebox[origin=c]{90}{\sf Relational Queries}}	&	\texttt{sql-01}	&		&	4	&	1	&	33	&	21	&	2	\tabularnewline	
	&	\texttt{sql-02}	&		&	2	&	1	&	16	&	3	&	1	\tabularnewline	
	&	\texttt{sql-03}	&		&	2	&	1	&	70	&	4	&	2	\tabularnewline	
	&	\texttt{sql-04}	&		&	3	&	2	&	7	&	9	&	6	\tabularnewline	
	&	\texttt{sql-05}	&		&	3	&	1	&	17	&	12	&	5	\tabularnewline	
	&	\texttt{sql-06}	&		&	3	&	2	&	9	&	9	&	9	\tabularnewline	
	&	\texttt{sql-07}	&		&	2	&	1	&	52	&	5	&	5	\tabularnewline	
	&	\texttt{sql-08}	&		&	4	&	3	&	206	&	6	&	2	\tabularnewline	
	&	\texttt{sql-09}	&		&	4	&	2	&	52	&	6	&	1	\tabularnewline	
	&	\texttt{sql-10}	&		&	3	&	2	&	734	&	10	&	2	\tabularnewline	
	&	\texttt{sql-11}	&		&	7	&	4	&	170	&	30	&	2	\tabularnewline	
	&	\texttt{sql-12}	&		&	6	&	3	&	32	&	36	&	7	\tabularnewline	
	&	\texttt{sql-13}	&		&	3	&	1	&	10	&	17	&	7	\tabularnewline	
	&	\texttt{sql-14}	&		&	4	&	3	&	23	&	11	&	6	\tabularnewline	
	&	\texttt{sql-15}	&		&	4	&	2	&	186	&	50	&	7	\tabularnewline	
	\bottomrule
\end{tabular}
\label{tbl:benchmarks}
\end{table}

\begin{figure}[t]
\center
\includegraphics[width=0.9\linewidth]{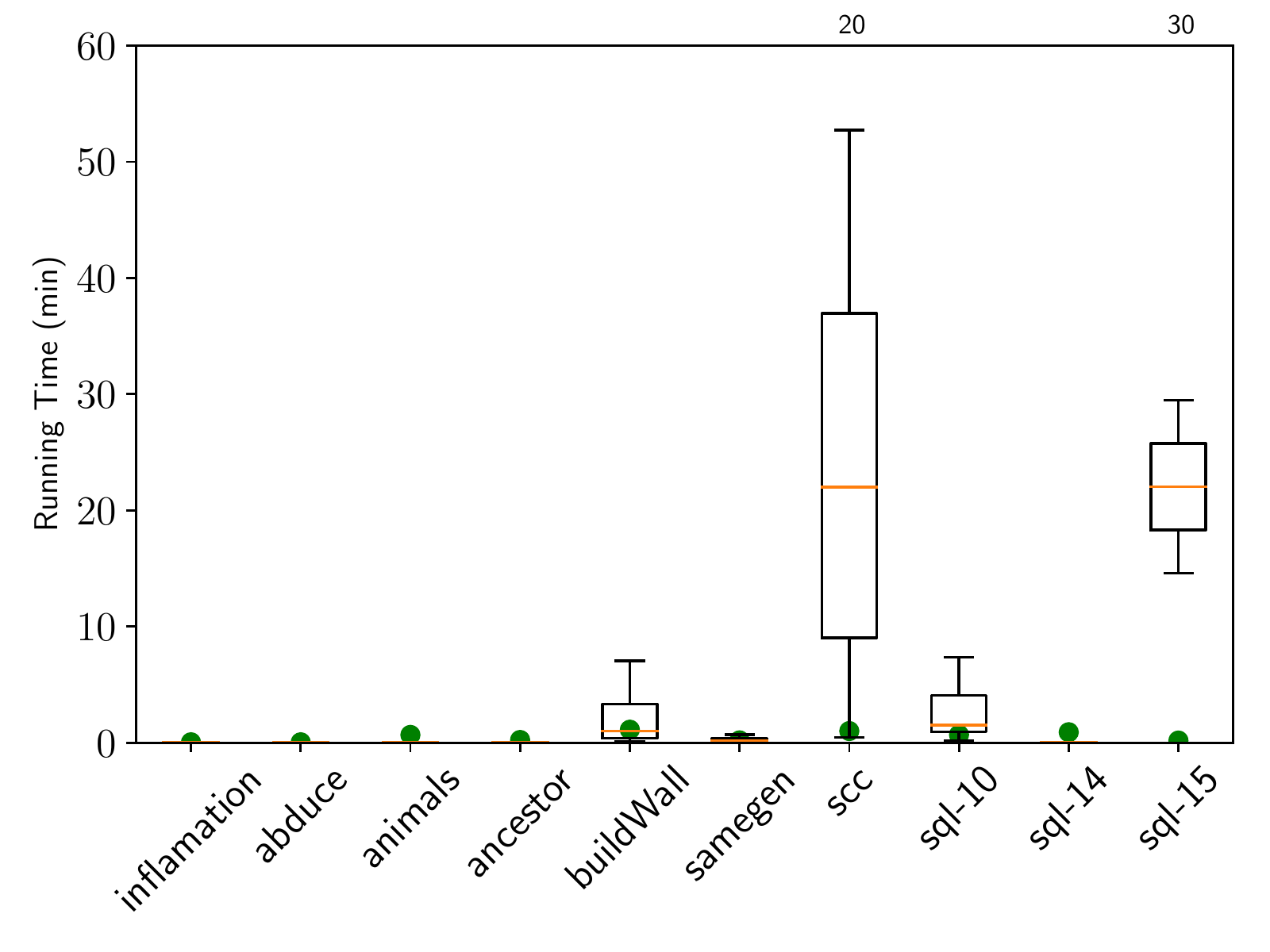}
\caption{Distribution of \Difflog{}'s running time from 32 parallel runs. The numbers on top represents the number of timeouts. Green circles represent the running time of \alps.}
\label{fig:appendix:distribution}
\end{figure}

\end{document}